\documentclass[runningheads]{llncs}
\usepackage{graphicx}
\usepackage{algorithm}
\usepackage{enumitem}
\usepackage{multicol}
\usepackage[noend]{algorithmic}

\usepackage{booktabs}
\usepackage{tabularx}  
\usepackage{xcolor}  
\usepackage{tikz}
\usetikzlibrary{calc}
\usetikzlibrary{positioning}
\usepackage{amsfonts}

\usepackage{hyperref}
\hypersetup{
    colorlinks,
    linkcolor=blue, 
    citecolor=blue, 
    urlcolor=blue   
}
\usepackage{amsfonts}
\usepackage{macros}
\usepackage{amsmath}
\usepackage{bbm}
\usepackage{caption}
\usepackage{subcaption}
\usepackage{soul}
\usepackage{todonotes}
\usepackage{marvosym}
\begin{document}
\title{SMT-Based Dynamic Multi-Robot \\ Task Allocation} 
%
%
\author{Victoria Marie Tuck\inst{1}(\Letter) \and
Pei-Wei Chen\inst{1}\thanks{Denotes significant contribution}(\Letter) \and
Georgios Fainekos\inst{2} \and
Bardh Hoxha\inst{2} \and
Hideki Okamoto \inst{2} \and
S. Shankar Sastry \inst{1} \and
Sanjit A. Seshia \inst{1}}
\authorrunning{V. Tuck et al.}
%
\institute{UC Berkeley, Berkeley CA 94704, USA \email{\{victoria\_tuck, pwchen, shankar\_sastry, sseshia\}@berkeley.edu} \and Toyota Motor North America, Research $\&$ Development, Ann Arbor MI 48105, USA \email{\{georgios.fainekos, bardh.hoxha, hideki.okamoto\}@toyota.com}} 
\maketitle              
\begin{abstract}
Multi-Robot Task Allocation (MRTA) is a problem that arises in many application domains including package delivery, warehouse robotics, and healthcare. In this work, we consider the problem of MRTA for a dynamic stream of tasks with task deadlines and capacitated agents (capacity for more than one simultaneous task). Previous work commonly focuses on the static case, uses specialized algorithms for restrictive task specifications, or lacks guarantees. We propose an approach to Dynamic MRTA for capacitated robots that is based on Satisfiability Modulo Theories (SMT) solving and addresses these concerns. We show our approach is both sound and complete, and that the SMT encoding is general, enabling extension to a broader class of task specifications. We show how to leverage the incremental solving capabilities of SMT solvers, keeping learned information when allocating new tasks arriving online, and to solve non-incrementally, which we provide runtime comparisons of. Additionally, we provide an algorithm to start with a smaller but potentially incomplete encoding that can iteratively be adjusted to the complete encoding. We evaluate our method on a parameterized set of benchmarks encoding multi-robot delivery created from a graph abstraction of a hospital-like environment. The effectiveness of our approach is demonstrated using a range of encodings, including quantifier-free theories of uninterpreted functions and linear or bitvector arithmetic across multiple solvers.

\keywords{Multi-Robot Task Allocation \and
Satisfiability Modulo Theories \and
Capacitated Robots \and
Incremental Solving \and
Cyber-Physical Systems \and
Robotics}
\end{abstract}

\section{Introduction}

Multi-robot systems have the potential to increase productivity by providing point-to-point pickup and delivery services, referring to the assignment to a team of robots of pickup and drop-off locations for transporting items under some optimization criteria and constraints.
Such services have already revolutionised warehouse management \cite{WurmanDAM2008aimag} by eliminating long travel times between locations for workers. New mobile robot systems are being developed for point-to-point pickup and delivery in environments where human-robot interaction is more likely -- such as in healthcare facilities \cite{JeonLK2017,DasEtAl2015jirs,SchueleEtAl2022hri}. 
Even though multi-robot systems in warehouses and healthcare settings share many similarities, the latter require a higher level of assurance. Formal methods provides this level of assurance by finding an assignment to robots if and only if one exists.


In this paper, we study the application of Satisfiability Modulo Theories (SMT)~\cite{barrett-smtbookch21} to the Multi-Robot Task Assignment (MRTA) \cite{ChakraaGLL2023ras,NUNES201755} problem for point-to-point pickup and delivery. We assume that the robots can execute only a limited number of concurrent tasks and assume that tasks are generated online, have a strict deadline, and each require only one robot. Although, as often used in warehouse settings, a gridworld abstraction may be applicable in a healthcare environment, the dynamic nature of the environment is better addressed through coarser abstractions (regions) within which local motion planning can be employed \cite{ShahS2022aamas,UwacuEtAl2022ral}. We therefore represent the environment by a weighted graph with nodes representing regions in the environment and weights the worst case cost (time) to move from region to region without any dynamic obstacles. Dynamic obstacles operating at shorter time scales, e.g., humans walking, could be avoided using local motion planning with safety guarantees \cite{MajdEtAl2021iros,ParwanaEtAl2023arxiv}. Note that this work focuses on solving the high-level planning part and leaves local motion planning and collision avoidance to downstream planners.

Satisfiability Modulo Theories (SMT)~\cite{barrett-smtbookch21} is a generalization of the Boolean satisfiability problem that answers the question whether a formula in first-order logic with background theories is satisfiable. Problems can be encoded as SMT formulas and passed to SMT solvers to determine satisfiability. Such solvers are widely used in industrial-scale applications (e.g.,~\cite{rungta2022billion}).
Given the progress in SMT solving,
our aim is to study the feasibility and scalability of solving the aforementioned instance of the MRTA problem with an SMT formulation. 
There are several reasons for investigating an SMT approach to MRTA: 1) at its core, MRTA is a combinatorial problem with arithmetic constraints, 2) an SMT formulation can be easily adapted to handle different variants of the MRTA problem \cite{ChakraaGLL2023ras,NUNES201755}, e.g., with complex task dependencies \cite{HekmatnejadPF2019case}, and 
3) even satisfying solutions without guaranteed optimality are relevant for this application since hierarchical planning methods \cite{ShahS2022aamas,UwacuEtAl2022ral} can refine a non-optimal high level plan to an optimal  (with respect to distance traveled) local motion plan.

Our contributions are:
\begin{itemize}
    \item a general, SMT-based framework leveraging quantifier-free theories of uninterpreted functions and bitvector or linear arithmetic for dynamic, capacitated MRTA via incremental solving;
    \item an approach to manage complexity by dynamically changing the number of free variables to fit the needs of the problem;
    \item theoretical results of completeness and soundness;
    \item and an experimental analysis of the runtime of our approach across different solvers (cvc5, Z3, and Bitwuzla) for a series of static (one set of tasks) and dynamic (tasks arriving online) benchmarks and showing that solver and setting used affects whether or not incremental solving is beneficial.
    
\end{itemize}

\section{Related Work}


Multi-Robot Task Allocation (MRTA) refers to the class of problems that encompasses many variants of the point-to-point pickup and delivery scheduling and path planning for multi-robot systems. 
For example, pickup and delivery tasks may have deadlines, robots may have to form a team to complete the task, or each robot may have different capacity constraints. For a detailed taxonomy for task allocation problems with temporal constraints, please refer to \cite{NUNES201755}.

Most heuristic-based MRTA algorithms search for a plan for each robot such that all tasks are completed while minimizing some objective. \cite{lopes2016simple} employs a hybrid genetic algorithm where local search procedures are used as mutation operators to solve for tasks with an unlimited fleet of capacity-constrained robots. \cite{sarkar2018scalable} uses a nearest-neighbor based approach to cluster nearby nodes and constructs routes for each of the clusters by mapping it to a traveling salesman problem (TSP) while minimizing overall package delivery time. \cite{chen2021integrated} adopts a marginal-cost heuristic and a meta-heuristic improvement strategy based on large neighborhood search to simultaneously perform task assignment and path planning while minimizing the sum of differences between the actual complete time and the earliest complete time over all tasks. 
The above heuristic-based approaches are often scalable on problems with up to 2000 tasks, but are not able to provide completeness guarantees, i.e., with hard deadlines. 
Moreover, heuristics are often tightly tied to a specific problem setting which makes it non-trivial to extend to other settings.
For a recent review on state-of-the-art optimization-based approaches to the MRTA problem, we direct readers to Chakraa et al. \cite{ChakraaGLL2023ras}.

In contrast, many approaches that provide strong guarantees do not scale well. \cite{okubo2022simultaneous} formulates the MRTA problem as a Mixed-Integer Linear Programming (MILP) problem to simultaneously optimize task allocation and planning in a setting where capacity-constrained robots are assigned to complete tasks with deadlines in a grid world. However, the proposed method suffers in computational performance when problem size grows large -- the approach is able to handle 20 tasks with 5 agents but the execution time is unavailable in the paper. 
\cite{jeon2016vehicle} promises globally optimal solution in a hospital setting by exhaustively searching through possible combinations of locations of interest and choosing capacity-constrained robots with minimum travel distances to complete tasks. The approach is able to solve 197 periodic tasks over a duration of 8 hours, but the runtime information for each solve is unknown. \cite{gavran2017antlab} is an linear-time temporal logic-based approach that provides strong guarantees and appears to scale well. However, its reliance on temporal logic may impact its ability to scale when length of the plans is large in time, whereas we represent time abstractly. Additionally, their approach does not allow for assigning robots new tasks before they have finished previous tasks, which our structure supports.

In this work, we are interested in the specification satisfaction problem, which is similar to the problem of minimizing cost where the cost goes to infinity if any of the constraints are not satisfiable. In contrast to the heuristic-based approach, our proposed approach is able to give completeness guarantees if solutions do not exist, while still achieving superior performance compared to those that give strong guarantees.



\noindent 
\begin{tabularx}{\textwidth}{lX}
    \toprule
    Symbol & Description \\
    \midrule
    $\taskset_j$ & The $j$th set of tasks in a task stream $\taskstream$ \\
    $\tasksetarrivaltime$ & Arrival time for the $j$th set of tasks \\
    $\taskset$ & Cumulative set of tasks until current time \\
    $\taskid$ & Task id for a task \\
    $\indexedtaskarrivaltime$ & Arrival time for task $\task$ \\
    $\taskdeadline$ & Deadline for task $\task$ \\
    $\actionsequence$ & Action sequence of an agent $\agent$ \\
    $\prefix$ & Prefix of length k of agent $\agent$'s action sequence \\
    $\sequenceelement$ & The $k$th element of agent $\agent$'s action sequence \\
    $\plan$ & A plan for a set of agents $\agentset$ and set of tasks $\taskset$ \\
    $\dptuple$ & Action tuple for agent $\agent$ at action point $\dpvar$ \\
    $Loc(\dpid)$ & Converts an action id to a location id \\
    $\assumelist$ & List of assumption vars limiting number of available action points \\
    $\rho$ & Time taken to pick up/drop off items \\
    $D_{min}, D_{max}$ & Min/max number of action points needed by task set $\taskset$ \\
    $\taskstarttime, \taskendtime$ & Start/end time for task $\task$ \\
    $\taskagent$ & Agent that completes task $\task$ \\
    $\capacity$ & Each agent's capacity for tasks at once \\
    $\encoding_j$ & Encoding at the $j$th iteration of the algorithm \\
    $\move, \wait, \pick, \drop$ & Move, wait, pick, drop actions \\
    $\location \in \locationset$ & Location in a set of system locations \\
    $\locationid \in \locationidset$ & Location id in set of location ids corresponding to the location set \\
    $\duration_k(\actionsequence)$ & Duration of a prefix of length $k$ of agent $\agent$'s action sequence \\
    $\sequenceload_k(\actionsequence)$ & Load of an agent $\agent$ at the $k$th element in their action sequence \\
    
    \bottomrule
\end{tabularx}
\section{Problem Formulation}

We use $\mathbb{Z}_{++}$ to denote the set of strictly positive integers.

\subsection{Workspace Model}

We assume we are given a finite set of designated system locations $\location \in \locationset$ each with a unique id $\locationid \in \locationidset \subset \mathbb{Z}_+$ where $\location \in \rtwo$. For example, each system location $\location$ is a spot in a building where a robot can start, pick up an object, or drop off an object. We are given a complete, weighted, undirected graph $\graph = (V, E)$ 
where $V=\locationidset$ and $E=\{(\locationid_i, \locationid_j, w_{i,j}) | \locationid_i \in \locationidset, \locationid_j \in \locationidset, w_{i,j} \in \mathbb{Z}_+\}$ where $w_{i,j}$ is 0 \textit{if and only if} $\locationid_i=\locationid_j$. The weight of the edge between vertices $\locationid_i$ and $\locationid_j$ is $w_{i,j}$. This weight $w_{i,j}$ denotes the travel time between the two points, which satisfies the triangle inequality with respect to all other sites $\location_k$, i.e., $w_{i,k} + w_{k,j} \geq w_{i,j} \ \forall \locationid_k \in \locationidset$.




\subsection{System Model}

A task $\task$ is a tuple $\tasktuple$ where $\taskid \in \mathbb{Z}_+$ is the task's unique id, $\taskstartloc, \taskendloc \in \locationidset$ are the starting and ending location ids, respectively, and $\taskarrivaltime, \taskdeadline \in \mathbb{Z}_{+}$ are the arrival time and deadline, respectively. $i$ and $f$ stand for initial and final, respectively. Each task is to move a corresponding object, which takes up one unit of capacity. 

Sets of tasks arrive as a sequence of incoming tasks $\taskstream$ called the task stream. Each entry of the task stream is a tuple $\indexedtasksettuple$. The first entry $\taskset_j$ of the tuple is an ordered set of tasks, and the second entry is the arrival time of the set. The arrival time $\taskarrivaltime$ for any task in the set is the same as the set's arrival time  $(\tasksetarrivaltime = \taskarrivaltime \ \forall \task \in \indexedtaskset)$. 
We use the same time in both contexts to more easily reference the arrival time depending on if we are reference a task in the set or the entire set. Assume $\tasksetarrivaltime[0]=0$. We assume the stream is finite with a known total number of tasks $\numtasks_{max}$. We require that this sequence is monotonically increasing with respect to the second element of $\indexedtasksettuple$. Let $\indexedcumtaskset = \bigcup \nolimits_{j'=0}^j \taskset_{j'}$ be the total set of tasks that have arrived. We constrain the task id of the first task to be zero and all following tasks to have ids that increment by one. More formally, $\forall \task \in \indexedtaskset$, $\taskid \in \{|\cumtaskset_{j-1}|, \ldots, |\indexedcumtaskset|-1\}$ with $\cumtaskset_{-1} = \emptyset$. We use $\numtasks = |\indexedcumtaskset|$ for the current total number of tasks where $\numtasks$ will change with the context. We will sometimes notate a set of tasks as $\taskset$. We require that the unique task ids start at 0 and increase by 1 for every new task that arrives.

There exists a finite, zero-indexed, ordered set $\agentset$ of $\numagents$ agents. Each agent $\agent$ has a unique id $\agentid \in \{0, \ldots, N-1\}$ and a starting position $\agent_{\location} \in \locationset$ that may not be unique. Each agent has a capacity $\capacity$ for tasks at one time.

We define a set of actions $\actionsset$ that a robot can take as $\actionsset = (\move, \locationid) \cup (\{\pick, \drop\},$
$\taskid)$ $\cup (\wait, \waittime)$. 
The action $(\move, \locationid)$ designates that the robot is moving to location $\location$ with id $\locationid$, $(\{\pick, \drop\}, \taskid)$ designates that the robot is picking up ($P$) or dropping off ($D$) the task $m$ with id $\taskid$, and $(\wait, \waittime)$ has the agent wait in the same position for a time $\waittime \in \mathbb{Z}_{++}$. We use the short-hand that $\pick_{\mu} = (\pick, \mu)$ and similar for drop. Each of $\{\moveset, \waitset, \pickset, \dropset\}$ represent all actions of that type, e.g. $\moveset = \{(M, \locationid) | \locationid \in \locationidset$\}. We assume the pick and drop actions each take a pre-specified time of $\rho \in \mathbb{Z}_{++}$.


The following definitions are used to define types of plans and the goals of our algorithm. Figure \ref{fig:example} shows the input, example output plans, and an example workspace. We use $\prefix$ to denote the prefix of length $\sequenceindex$ of agent $\agent$'s action sequence. $\indicatorfunc{\alpha}$ is an indicator function that is $1$ when $\alpha$ is true and $0$ otherwise.

\begin{figure}
    \centering
    \includegraphics[width=\linewidth]{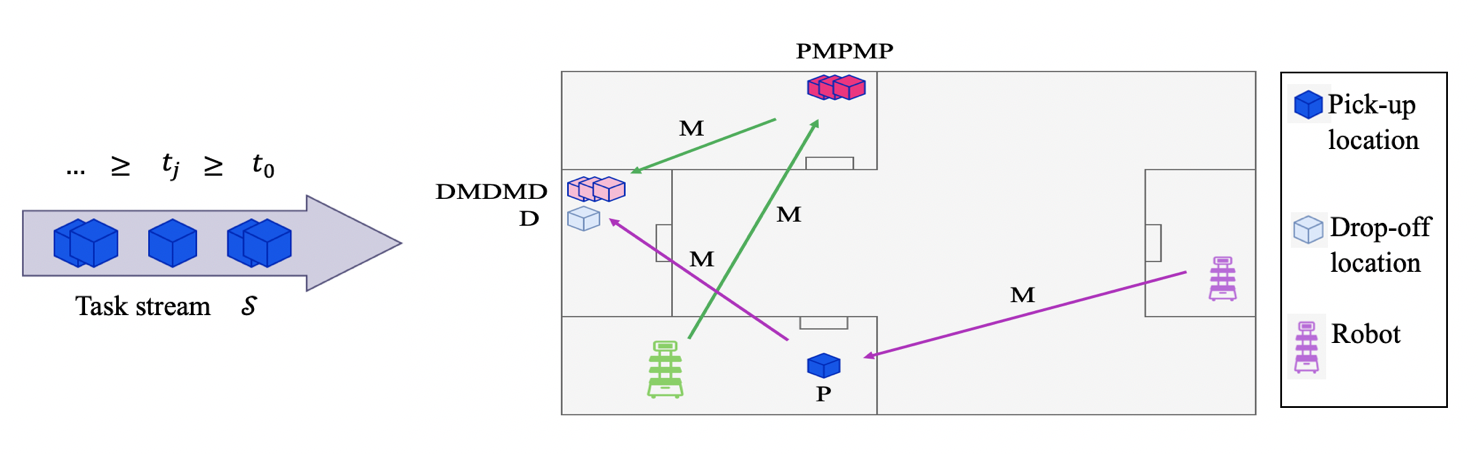}
    \caption{A task stream of sets of tasks arrives with monotonically increasing arrival times. Five system sites are shown. Example tasks and robot paths are shown on the right with $\pick$, $\drop$, and $\move$ used to represent the actions succinctly. The result action sequence for the right robot is (M,P,M,D) and (M,P,M,P,M,P,M,D,M,D,M,D). The moves between picks or drops at the same location are used to keep a consistent structure in the plan but take no time.}
    \label{fig:example}
\end{figure}

\begin{definition}
    \textbf{Action sequence.} An action sequence $\actionsequence$ is a finite sequence beginning with element $k=1$ where each element $\sequenceelement \in \actionsequence$ is an action ($\sequenceelement \in \actionsset$).
\end{definition}

\begin{definition}
    \textbf{Plan.} A \textit{plan} $\plan$ for the set of agents $\agentset$ and set of tasks $\taskset$ is an action sequence $\actionsequence$ with length $\sequenceindex_{\agent}$ for each agent $\agent$.
\end{definition}

\begin{definition}
    \textbf{Duration of an action sequence.} We compute the duration of a prefix $\prefix$ of $\actionsequence$ as 
    $$\duration_{\sequenceindex}(\actionsequence) = \Sigma_{l = 1}^{k} w_{\sigma_{l-1}, \sigma_{l}} \indicatorfunc{\sequenceelement[l] \in \moveset} + \rho ( \indicatorfunc{\sequenceelement[l] \in \pickset} + \indicatorfunc{\sequenceelement[l] \in \dropset}) + \waittime_{l} \indicatorfunc{\sequenceelement[l] \in \waitset}$$ 
    where $\sigma_0 = \agent_s$ and $\sigma_l$ and $\waittime_l$ are the location ids and times, respectively, of action $l$. $\sigma_{l}=\taskstartloc$ for a pick action and $\taskendloc$ for a drop action. $\sigma_{l}$ for a wait is the most recent location.
\end{definition}

\begin{definition}
    \textbf{Load of an action sequence.} We compute the load of a prefix $\prefix$ of $\actionsequence$ as $\sequenceload_k(\actionsequence) = \Sigma_{j=1}^k (\indicatorfunc{\sequenceelement[j] \in \pickset} - \indicatorfunc{\sequenceelement[j] \in \dropset})$.
\end{definition}

\begin{definition}
    \textbf{Consistent action sequence.} A plan $\plan$ is made of consistent action sequences $\allconsistentpredicate(\plan)$ if for each agent's sequence it starts with a move or wait; no capacity constraints are violated; any pick ($\pick_{\task} = (\pick, \taskid)$) and drop ($\drop_{\task} = (\drop, \taskid)$) actions for a task $\task$ are immediately preceded by the move to that point (A move that will require no time is still added if the agent "moves" to its current location.); pick precedes drop; drop follows pick; no two moves occur in a row; and any object in a sequence is only picked and dropped once ($\consistentpredicate$). An empty sequence is consistent. 
    \begin{align}
        \phi_{1}(\actionsequence) = &  \bigg(\actionsequence = \emptyset \bigg) \ \bigvee  \\
        &\bigg( \bigg(\sequenceelement[1] \in \moveset \cup \waitset \bigg) \wedge \bigg( \forall \kappa=1, ..., \sequenceindex_{\agent} \ (0 \leq \seqcapacity_{\kappa}(\actionsequence)) \wedge (\seqcapacity_{\kappa}(\actionsequence) \leq \capacity) \bigg) \notag
        \\
        & \wedge \bigg(\sequenceelement[\sequenceindex_{\agent}] \notin \moveset \bigg) \wedge \bigg(\forall \kappa=1, ..., \sequenceindex_{\agent}-1 \ \sequenceelement[\kappa] \in \moveset \Rightarrow \sequenceelement[\kappa + 1] \notin \moveset \bigg) \notag
        \\
        & \bigwedge \bigg(\forall \task \in \taskset, \ \forall \kappa=2, ..., \sequenceindex_{\agent} \  \bigg( (\sequenceelement[\kappa] = \pick_{\task}) \Rightarrow ((\prevsequenceelement[\kappa - 1] = \move_{\location_{i,m}}) \notag \\
        & \wedge (\bigvee \nolimits_{\kappa' > \kappa} \sequenceelement[\kappa'] = \drop_{\task}) \wedge (\forall \kappa' \neq \kappa \in 1, \ldots  \sequenceindex_{\agent}, \ \sequenceelement[\kappa'] \neq \pick_{\task})) \bigg) \notag \\
        &  \wedge \bigg( (\sequenceelement[\kappa] = \drop_{\task}) \Rightarrow ((\prevsequenceelement[\kappa - 1] = \move_{\location_{f,m}}) \wedge (\bigvee \nolimits_{\kappa' < \kappa} \sequenceelement[\kappa'] = \pick_{\task}) \notag \\
        & \wedge (\forall \kappa' \neq \kappa \in 1, \ldots  \sequenceindex_{\agent}, \ \sequenceelement[\kappa'] \neq \drop_{\task})) \bigg) \bigg) \bigg) \notag
        \\
        \allconsistentpredicate(\plan) = & \forall \actionsequence \in \plan, \ \consistentpredicate(\actionsequence) 
    \end{align}
\end{definition}

\begin{definition}
    \textbf{Completed task.} Tasks with id $\taskid$ in taskset $\taskset$ are completed $\allcompletedpredicate[\taskset]$ by a plan $\plan$ if for each task there exists a single agent $\agent$ with a consistent action sequence $\actionsequence \in \plan$ that picks and drops the action. Additionally, the drop action must be before the deadline, and the time before moving to the pick action must be greater than or equal to the start time $\taskarrivaltime$ ($\completepredicate$).
    We define the predicates $\phi_{pickup}(\task, \actionsequence)$ and $\phi_{dropoff}(\task, \actionsequence)$ as
    \begin{align}
        \phi_{pickup}(\task, \actionsequence) = & (\exists \kappa_{p} \in [1, \sequenceindex_{\agent}] \text{ st. } (\sequenceelement[\kappa_{p}] = \pick_{\task})) \\
        \phi_{dropoff}(\task, \actionsequence) = & (\exists \kappa_{d} \in [1, \sequenceindex_{\agent}] \text{ st. }(\sequenceelement[\kappa_{d}] = \drop_{\task}))
    \end{align}
    \begin{align}
        \completepredicate(\plan) = &  \ \exists \actionsequence \in \plan
        \\
        st. \ & \consistentpredicate(\actionsequence) \wedge \phi_{pickup}(\task, \actionsequence) \wedge \phi_{dropoff}(\task, \actionsequence) \notag \\
        & \wedge \forall \agent' \in \agentset \text{st.} \ \agent' \neq \agent, \ \neg \phi_{pickup}(\task, \otheractionsequence) \wedge \neg \phi_{dropoff}(\task, \otheractionsequence) \notag 
        \\
        & \wedge (\duration_{\kappa_{d}}(\actionsequence) \leq \taskdeadline) \wedge (\duration_{\kappa_{p-2}}(\actionsequence) \geq \taskarrivaltime) \notag
        \\
        \allcompletedpredicate[\taskset](\plan) = & \forall \task \in \taskset, \ \completepredicate(\plan)
    \end{align}
\end{definition}

\begin{definition}
    We define $\Pi \models \Phi$ to hold when $\Phi(\Pi)$ is true.
\end{definition}

\begin{definition}
    \textbf{Valid Plan.} We define a valid plan $\plan$ for the set of agents $\agentset$ and set of tasks $\taskset$ as one where each agent's $\actionsequence$ is consistent and all tasks $\task \in \taskset$ are completed. An empty plan is considered valid when $\taskset = \emptyset$. This can be written as $\plan \models \allconsistentpredicate \wedge \allcompletedpredicate[\taskset]$.
\end{definition}

\begin{definition}
    \textbf{Updated plan.} A valid plan $\hat{\plan}$ with agent sequence lengths $\hat{\sequenceindex_{\agent}}$ is updated at a time $t$ from a previous valid plan $\plan$ with agent sequence lengths $\sequenceindex_{\agent}$ if the following conditions hold. In each agent's action sequence there is an equivalent prefix $\primeprefix[\kappa]$ to that in $\plan$ where the prefix is all of $\actionsequence$ in $\plan$ followed by a wait action or $(\sequenceelement[\kappa] \in \pickset \cup \dropset) \wedge \duration_{\sequenceindex}(\actionsequence) \geq t$. This means that past actions and the current action are unchanged. A plan can only be updated at a system location $\location$. A valid initial plan $\hat{\plan}$ is always considered updated from an empty previous plan $\plan$: $\forall \agent \in \agentset, \actionsequence = ()$. This empty plan will sometimes be notated as $\plan_{-1}$. We also require that each agent's action sequence is efficient and does not contain extra waits. 
    \begin{align}
        \updatedpredicate(\hat{\plan}) = & \ \forall \agent \in \agentset, \bigg((\prefix[\sequenceindex_{\agent}] = \primeprefix[\sequenceindex_{\agent}]) \\
        & \wedge \bigg( (\actionsequence = \primeactionsequence) \vee (\hat{s}^{\sequenceindex_{\agent}+1}_{\agent} = (\wait, t - \duration_{[\sequenceindex_{\agent}]}(\actionsequence))) \bigg) \bigg) \notag
        \\
        & \bigvee \ \bigg( \bigg( \exists \kappa = 1, \ldots, \sequenceindex_{\agent}. \ (\prefix[\kappa] = \primeprefix[\kappa]) \wedge (\duration_{\kappa}(\actionsequence) \geq t) \wedge (\sequenceelement[\kappa] \in \pickset \cup \dropset ) \bigg) \notag \\ 
        & \wedge \bigg(\forall \kappa \ \text{st.} \ \duration_{\kappa}(\hat{\actionsequence}) \geq t, \sequenceelement[\kappa] \notin \waitset \bigg) \bigg) \notag
    \end{align}
\end{definition}

\begin{definition} \label{def:soundness}
    \textbf{Soundness.} Let $\requirement = \allconsistentpredicate \wedge \allcompletedpredicate \wedge \indexedupdatedpredicate[\tasksetarrivaltime][\plan_{j-1}]$. An algorithm is \textit{sound} for a given finite task stream of length $\taskstreamlength$ if $ \ \forall j=0, \ldots, \taskstreamlength, \ (result_j = sat) \Rightarrow (\plan_j \models \requirement)$.
\end{definition}

\begin{definition} \label{def:completeness}
    \textbf{Completeness.} An algorithm is complete for a given finite task stream of length $\taskstreamlength$, if $ \ \forall j=0, \ldots, \taskstreamlength, \ (\plan_{j} \models \requirement[j]) \Rightarrow (result_{j} = sat)$.
\end{definition}

\subsection{Problem Statement}

Given a set of agents $\agentset$, task stream $\taskstream$ including a known number of total tasks, and travel time graph $\graph$, after each element $j$ in the task stream arrives at $t_j$, find a valid plan $\plan_j$ updated from a previous plan $\plan_{j-1}$ if one exists.

\section{Summary of Approach}

\subsection{Preliminaries} 
We assume a basic understanding of propositional and first-order logic. \textit{Satisfiability Modulo Theories} (SMT), a generalization of the Boolean satisfiability problem, is the satisfiability problem for formulas with respect to a first-order theory, or combinations of first-order theories. SMT solvers, such as Z3~\cite{de2008z3} and cvc5~\cite{DBLP:conf/tacas/BarbosaBBKLMMMN22}, are used to solve SMT formulas, where a model is returned if the SMT formula is satisfiable or otherwise reports unsatisfiability. 
Below we briefly introduce three theories of interest in the paper.

The theory of Equality Logic and Uninterpreted Functions (\textbf{EUF}) introduces the binary equality (=) predicate and uninterpreted functions, which maintains the property of functional congruence stating that function outputs should be the same when function inputs are the same.
The theory of BitVectors (\textbf{BV}) incorporates fixed-precision numbers and operators, e.g. the bitwise AND operator. Note that the modulus operator with a constant 2 can be replaced by a bitwise AND operation using a constant 1.
The theory of Linear Integer Arithmetic (\textbf{LIA}) introduces arithmetic functions and predicates and constrains variables to only take integer values. Note that $v = x \mod k$, where $x$ is a variable and $k$ and $v$ are constants, can be translated into $x = k \cdot n + v$ with a fresh integer variable $n$. In this paper, we only consider quantifier-free (QF) SMT formulas, and abbreviate Quantifier-Free Bit Vector Theory (Linear Integer Arithmetic) with Uninterpreted Functions as QF\_UFBV (QF\_UFLIA). 

\subsection{Encoding Literals}

We introduce the notion of an \textbf{action id} $i$ to describe what action an agent is taking. The first $\numagents$ action ids represent going to the corresponding agent number's starting position. For each following pair of values $j$, the values correspond to picking up and dropping off the $j$th object, respectively. 

Each agent $\agent$ in $\agentset$ is allowed $\numdp$ action tuples where a tuple for an action point $d$ = $\dptuple$. $\dpid$ is the id of the action being taken, $\dptime$ is the time by which the action corresponding to the action id $\dpid$ has been completed, and $\dpload$ is the agent's load at the time $\dptime$ after completing the action with id $\dpid$.

For each task $\task \in \taskset$, we define the task start $\taskstarttime$ and end $\taskendtime$ times with agent $\agent$ completing the task $\taskagent$. This creates a tuple $(\taskstarttime, \taskendtime, \taskagent)$.

\subsection{Incremental Solving}
For tasks arriving online, we must re-solve our SMT problem given the new tasks. We implement this using push and pop functionalities provided by SMT solvers to retain information about previous solves. We push constraints onto the stack and pop them before adding new tasks as explained in AddTasks(). Past action points are constrained to be constant in SavePastState().

We use assumptions-based incremental solving to adjust the number of action points used. This allows us to force the solver to start with the minimum number necessary and search for the sufficient number needed for a sat result if one exists. We assume we are either given a list $\dplist$ of action point counts or will construct a reasonable list by starting with the $D_{min}(\taskset) = 2\lceil\frac{|\taskset|}{\numagents} \rceil$ and incrementing by two to $D_{max}=2\numtasks_{max} + 1$ inclusive. The last value of the user provided list must be $D_{max}$. More about $D_{max}$ is explained in section \ref{sec:theoretical}. From this, we define an assumption list $\assumelist$ of booleans the same length as $\dplist$. An element $\assumelist[k]$ is used as an input to the solver to designate how many action points to use. $D = D_{max}$ in the encoding to allow for using the max number of action points if necessary to find a satisfying assignment.

\subsection{SMT Encoding}

Fig. \ref{fig:encoding} includes the types of constraints that are used in our encoding. We include all for the initial solve and then some are iteratively removed and added for additional solves as explained in Algorithm \ref{alg:overall}. Assume $t_{max}$ is set large enough to be greater than the maximum deadline of any task to appear plus maximum travel time between any two points. For our completeness argument, we will assume that $\numdp = \numdp_{max} = 2\numtasks_{max} + 1$.

For readability, we define the following symbols:
\begin{align}
    \pickup \ = & \ ITE((\dpid \ \& \ 1 = mod(\numagents, 2)) \wedge (\dpid > \numagents), 1, 0) \label{eq:pick}\\
    \dropoff \ = & \ ITE((\dpid \ \& \ 1 = abs(mod(\numagents, 2) - 1)) \wedge (\dpid > \numagents), 1, 0) \label{eq:drop}\\
    \droplater \ = & \ ( {\bigvee \nolimits_{\dpvar' = \dpvar + 1}^{D - 1} \dpid[\agent][\dpvar'] = 2\taskid + \numagents + 1} ) \vee False \label{eq:droplater} \\
    \agentstarts \ = & \ \bigvee \nolimits_{\dpvar = 1}^{\numdp - 1} \dpid = 2\taskid + \numagents \label{ eq:agentstarts} \\
    \validid \ = & \ (\numagents \leq \dpid) \wedge (\dpid < 2\numtasks + \numagents) \label{eq:valid}
\end{align}

\ref{enc:init_and_allowed_dp} initializes the action tuples; \ref{enc:stay_home}, \ref{enc:load_tracking_and_doing_tasks}, and \ref{enc:doing_tasks} relate action tuple pairs; \ref{enc:room_UF_agents_and_dist_UF} restricts the uninterpreted functions, \ref{enc:bound_load_bound_id_go_home} and \ref{enc:at_home_or_doing_task_with_id} bound action point values, \ref{enc:room_UF_task_rooms} restricts the task uninterpreted functions, \ref{enc:started_task} - \ref{enc:assigned_agent} relate action and task tuples; and \ref{enc:valid_agent_and_times} restricts task tuples.
Note that $mod(N, 2)$ and $abs(mod(N, 2) - 1)$ are pre-computed and are within \{0, 1\}.

\begin{figure}[htb]
    \centering
    \begin{align}
        \baseencoding & (\agentset, \graph, \assumelist, \numdp, \capacity, t_{max}) = \notag \\
        \big(& \conjunctionoveragents \dptuple[\agent][0] = (\agentid, \initload, \inittime) \conjunctionoverdplist \assumelist[\dpindex] \Rightarrow (\dpid[\agent][\dplist[\dpindex]] = \agentid) \label{enc:init_and_allowed_dp} 
        \tag{E.\theencodingcounter}
        \\
        & \conjunctionoveragents \bigwedge \nolimits_{\dpvar = 1}^{\numdp - 2} (\dpid = \agentid) \Rightarrow (\nextdpid = \agentid) \label{enc:stay_home} 
        \stepcounter{encodingcounter} \tag{E.\theencodingcounter}
        \\
        & \conjunctionoveragents \conjunctionoveractionpoints (\dpload = \prevdpload + \pickup - \dropoff) \wedge (\neg(\dpid = \agentid) \Rightarrow (\neg(\prevdpid = \dpid))) \label{enc:load_tracking_and_doing_tasks}
        \stepcounter{encodingcounter} \tag{E.\theencodingcounter}
        \\
        & \conjunctionoverlocations[\locationid_1] \conjunctionoverlocations[\locationid_2] Dist(\locationid_1, \locationid_2) = w_{\locationid_1, \locationid_2} \conjunctionoveragents Loc(\agentid) = n_s \label{enc:room_UF_agents_and_dist_UF}
        \stepcounter{encodingcounter} \tag{E.\theencodingcounter}
        \\
        & \conjunctionoveragents
        \conjunctionoveractionpoints (\dpload \geq 0) \wedge (\dpload \leq \capacity) \wedge (\dpid \geq 0) \wedge (\dpid = \agentid) \Rightarrow (\dptime = t_{max}) \big) \label{enc:bound_load_bound_id_go_home}
        \stepcounter{encodingcounter} \tag{E.\theencodingcounter}
        \\
        \updateencoding & (\agentset, \delta, \tasksetarrivaltime, \numdp, \numtasks)  = \notag \\
        \big(& \conjunctionoveragents \bigwedge \nolimits_{\dpvar = \delta_{\agent}}^{\numdp - 1} (\dpid \geq \numagents) \Rightarrow \notag \\
        & (\dptime = ITE(\prevdptime \leq t_j, t_j, \prevdptime) + Dist(Loc(\prevdpid), Loc(\dpid)) + \rho) \label{enc:doing_tasks}
        \stepcounter{encodingcounter} \tag{E.\theencodingcounter}
        \\
        & \conjunctionoveragents \bigwedge \limits_{\dpvar = 1}^{\numdp - 1} \neg(\dpid = \agentid) \Rightarrow (\validid)\big) \label{enc:at_home_or_doing_task_with_id}
        \stepcounter{encodingcounter} \tag{E.\theencodingcounter}
        \\
        \taskencoding & (\agentset, \taskset, \numdp) = \notag \\
        \big(& \conjunctionovertasks (Loc(2\taskid + \numagents) = \taskstartloc) \wedge (Loc(2\taskid + \numagents + 1) = \taskendloc) \label{enc:room_UF_task_rooms} 
        \stepcounter{encodingcounter} \tag{E.\theencodingcounter}
        \\
        & \conjunctionoveragents \conjunctionoveractionpoints \conjunctionovertasks (\dpid = 2\taskid + \numagents) \Rightarrow ((\taskstarttime = \dptime) \wedge \droplater) \label{enc:started_task} 
        \stepcounter{encodingcounter} \tag{E.\theencodingcounter}
        \\
        & \conjunctionoveragents \conjunctionoveractionpoints \conjunctionovertasks (\dpid = 2\taskid + 1 + \numagents) \Rightarrow (\taskendtime = \dptime) \wedge (\taskagent = \agentid) \label{enc:ended_task} 
        \stepcounter{encodingcounter} \tag{E.\theencodingcounter}
        \\
        & \conjunctionoveragents \conjunctionovertasks (\taskagent = \agentid) \Rightarrow \agentstarts \label{enc:assigned_agent} 
        \stepcounter{encodingcounter} \tag{E.\theencodingcounter}
        \\
        & \conjunctionovertasks (\taskagent \geq 0) \wedge (\taskagent < \numagents) \wedge (\taskstarttime \geq \taskarrivaltime + \rho) \wedge (\taskendtime \leq \taskdeadline)\big) \label{enc:valid_agent_and_times} 
        \stepcounter{encodingcounter} \tag{E.\theencodingcounter}
    \end{align}
    \caption{Constraints used in the encoding developed in Algorithm \ref{alg:overall}}
    \label{fig:encoding}
\end{figure}

\subsection{Overall Algorithm \label{sec:algo}}

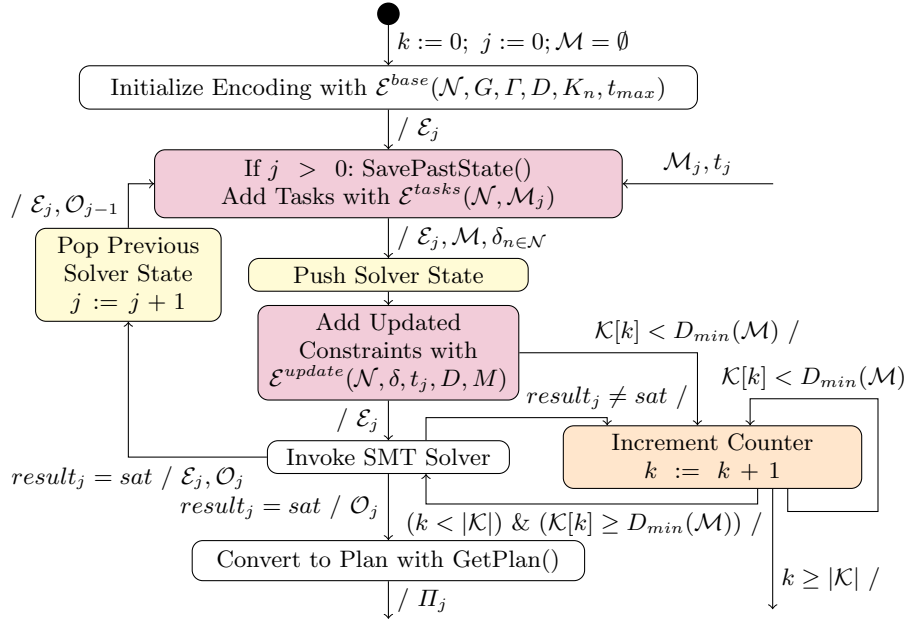
\begin{figure}
\centering
    \begin{tikzpicture}
        \node[draw, rectangle, fill=white!20, text width=8cm, align=center, rounded corners] (initNode) {Initialize Encoding with $\baseencoding(\agentset, \graph, \assumelist, \numdp, \capacity, t_{max})$};
        \tikzstyle{blackcircle} = [circle, fill=black, minimum size=0.3cm, inner sep=0pt]
        \node[blackcircle, above=15pt of initNode] (initcircle) {};
        \node[draw, rectangle, fill=purple!20, below=15pt of initNode, text width=6cm, align=center, rounded corners] (addTasksNode) {If $j>0$: SavePastState() \\ Add Tasks with $\taskencoding(\agentset, \taskset_j)$};
        \node[draw, rectangle, fill=yellow!20, below=15pt of addTasksNode, text width=3.5cm, align=center, rounded corners] (pushNode) {Push Solver State};
        \node[draw, rectangle, fill=purple!20, below=5pt of pushNode, text width=3.25cm, align=center, rounded corners] (reAddConstraintsNode) {Add Updated Constraints with $\updateencoding(\agentset, \delta, \tasksetarrivaltime, \numdp, \numtasks)$};
        \node[draw, rectangle, fill=white!20, below=15pt of reAddConstraintsNode, text width=3cm, align=center, rounded corners] (solveNode) {Invoke SMT Solver};
        \node[draw, rectangle, fill=orange!20, right=20pt of solveNode, text width=3.75cm, align=center, rounded corners] (assumptionNode) {Increment Counter \\ $k := k + 1$};
        \node[draw, rectangle, fill=white!20, below=25pt of solveNode, text width=5cm, align=center, rounded corners] (planNode) {Convert to Plan with GetPlan()};
        \node[draw, rectangle, fill=yellow!20, left=10pt of pushNode, text width=2.25cm, align=center, rounded corners] (removeConstraintsNode) {Pop Previous Solver State \\ $j := j + 1$};
        \draw[->] ($(initNode.north)+(0,0.6)$) -- (initNode.north) node[midway, right] {$k:=0; \ j:=0; \taskset = \emptyset$};
        \draw[->] (initNode.south) -- (addTasksNode) node[midway, right] {$\slash \ \encoding_j$};
        \draw[->] (addTasksNode) -- (pushNode) node[midway, right] {$\slash \ \encoding_j, \taskset, \delta_{\agent \in \agentset}$};
        \draw[->] ($(addTasksNode.east)+(2,0)$) -- (addTasksNode.east) node[midway, above] {$\taskset_j, t_j$};
        \draw[->] (pushNode) -- (reAddConstraintsNode) node[midway, right] {};
        \draw[->] (reAddConstraintsNode) -- (solveNode) node[midway, left] {$\slash \ \encoding_j$};
        \draw[->] (reAddConstraintsNode) -| ($(assumptionNode.north)+(-0.2,0)$) node[midway, above] {$\dplist[\dpindex] < \numdp_{min}(\taskset) \ \slash$};
        \draw[->] ($(solveNode.north)+(0.5,0)$) |- ++(0.3,0.3) -| ($(assumptionNode.north)+(-1.4,0)$) node[midway, above] {$result_j \neq sat \ \slash$};
        \draw[->] ($(assumptionNode.south)+(0.6,0)$) |- ++(-0.2, -0.2) -| ($(solveNode.south)+(0.5,0)$) node[near start, below] {$(k < |\dplist|) \ \& \ (\dplist[\dpindex] \geq \numdp_{min}(\taskset)) \ \slash$};
        \draw[->] (solveNode) -- (planNode) node[midway, left] {$result_j = sat \ \slash \ \modeloutput_j$};
        \draw[->] (solveNode.west) -| (removeConstraintsNode) node[midway, below] {$result_j = sat \ \slash \ \encoding_j, \modeloutput_j$};
        \draw[->] (removeConstraintsNode) |- (addTasksNode.west) node[near start, left] {$\slash \ \encoding_j, \modeloutput_{j-1}$};
        \draw[->] (planNode.south) -- ($(planNode.south)+(0,-0.5)$) node[midway, right] {$\slash \ \plan_j$};
        \draw[->] ($(assumptionNode.south)+(0.8,0)$) -- ($(assumptionNode.south)+(0.8,-1.6)$) node[near end, right] {$k \geq |\dplist| \ \slash $};
        \draw[->] ($(assumptionNode.south)+(1,0)$) |- ++(1, -0.3) -| ++(0.2, 1.5) -|($(assumptionNode.north)+(0.5,0)$) node[near start, above] {$\dplist[\dpindex] < \numdp_{min}(\taskset)$};
    \end{tikzpicture}
    \caption{Flow chart of Algorithm $\ref{alg:overall}$ starting at the black circle. $A \ \slash \ B$ denotes a conditional $A$ and a changed variable $B$. On the right is the assumption block to manage problem size. The third row from the top (in yellow) shows pushes and pops of the solver state.}
    \label{fig:alg_flowchart}
\end{figure}

Algorithm \ref{alg:overall} defines the overall algorithm. Figure \ref{fig:alg_flowchart} shows a flowchart of Algorithm \ref{alg:overall}. We define the following functions for use in the algorithm:
\begin{itemize}
    \item $\text{Push($\encoding$), Pop($\encoding$)}$: Refers to pushes and pops for incremental solving.
    \item $\text{Solve($\encoding, \assume$)}$: Solves the encoding $\encoding$ assuming $\assume \in \assumelist$ returning the Result $\in$ {sat, unsat, unknown} and satisfying assignment $\modeloutput$ if available.
    \item $\text{GetPlan}(\modeloutput)$: Taking each agent's action points in order, for each action point $d$ after the 1st where $\dpid \neq \agentid$, let $\taskid = \lfloor \frac{\dpid - \numagents}{2} \rfloor$. If $\dptime - \prevdptime > w_{Loc(\dpid), Loc(\prevdpid)} + \rho $, add $(W, \dptime - w_{Loc(\dpid), Loc(\prevdpid)} - \rho)$. Add $(M, \taskstartloc)$, $(P, \taskid)$ if $mod(\dpid - \numagents, 2) = 0$ and $(M, \taskendloc), (D, \taskid)$ if $mod(\dpid - \numagents, 2) = 1$.
    \item $\text{SavePastState($\tasksetarrivaltime, \encoding, \modeloutput$)}$: For each agent, loop from $d=0$ to $d=D-1$, calling Add($\dptuple = \modeloutput_{j-1}(\dptuple), \encoding, \emptyset, \emptyset$) until $(\modeloutput_{j-1}(\dptime) \geq \tasksetarrivaltime)\ | \ (\dpvar + 1 < \numdp \text{ \& }\nextdpid = n)$ at which point $d+1$ is saved into $\delta_{\agentid}$. Return the updated encoding and $\delta_{\agent \in \agentset}$.
\end{itemize}

\begin{algorithm}
\caption{Overall Algorithm}
\label{alg:overall}
\begin{algorithmic}[1]
    \STATE $\taskset, \encoding_0, \dpindex, \tasksetindex \gets \emptyset, \initialencoding, 0, 0$ \label{algline:init}
    \STATE $\encoding_j \gets \baseencoding(\agentset, \graph, \assumelist, \numdp_{max}, \capacity, t_{max})$ \label{algline:encodingbase}
    \STATE \textbf{WaitForTaskSet $\indexedtasksettuple$} \label{algline:wait}
    \IF {$j > 0$} \label{algline:ifj}
        \STATE $\encoding_j, \delta \gets$ SavePastState($\tasksetarrivaltime, \encoding_j, \modeloutput)$ \label{algline:savepaststate}
    \ENDIF \label{algline:endifj}
    \STATE $\encoding_j, \taskset \gets \encoding_j \cup \taskencoding(\agentset, \indexedtaskset), \taskset \cup \indexedtaskset$ \label{algline:addtasks} \COMMENT{Add constraints for new tasks}
    \STATE $Push(\encoding_j)$
    \STATE $\encoding_j \gets \encoding_j \cup \updateencoding(\agentset, \delta, \tasksetarrivaltime, \numdp_{max}, |\taskset|)$ \label{algline:finishaddtasks} \COMMENT{Update constraints due to new tasks}
    \WHILE {$\dpindex < |\dplist|$} \label{algline:whileactionpoints}
        \IF [Increment $k$ if num of action points] {$\dplist[\dpindex] < \numdp_{min}(\taskset)$} \label{algline:checkminactionpoints} 
            \STATE $\dpindex \gets \dpindex + 1$ \label{line:dp_increase1} \COMMENT{is smaller than minimum needed}
            \STATE \textbf{continue}
        \ENDIF \label{algline:endminactionpoints}
        \STATE $result_j$, $\modeloutput_j$ = Solve($\encoding_j, \assumelist[k]$) \label{algline:solve} \COMMENT{Invoke SMT solver}
        \IF {$result_j = sat$}
            \STATE $\plan_{\tasksetindex}$ = GetPlan($\modeloutput_j$)
            \STATE \textbf{break}
        \ENDIF
        \STATE $\dpindex \gets \dpindex + 1$\label{line:dp_increase2} \COMMENT{Unsat case -- more action points may be needed}
    \ENDWHILE \label{algline:endwhileactionpoints}
    \IF {$result_j \neq sat$}
        \STATE \textbf{return} unsat \label{algline:exitunsat}
    \ENDIF
    \STATE $\tasksetindex \gets \tasksetindex + 1$
    \STATE $\encoding_j \gets$ Pop($\encoding_{j-1}$) \label{algline:popoff}
    \STATE \textbf{GoTo} \ref{algline:wait} \label{algline:goto}
\end{algorithmic} 
\end{algorithm}

In Algorithm $\ref{alg:overall}$, lines \ref{algline:init} and \ref{algline:encodingbase} add constraints that do not use task or arrival time information. Lines in between \ref{algline:wait} and \ref{algline:goto} contain code run after waiting for each task set to arrive. For each task set other than the 0th, line \ref{algline:endifj} saves the state of action points that have occurred in the past or that an agent is currently completing. In lines \ref{algline:addtasks}-\ref{algline:finishaddtasks}, we add constraints relating to the incoming set of tasks, pushing as necessary so that constraints can be popped later. The constraints that are pushed and popped are those that are based on new information ($\tasksetarrivaltime$ and $\indexedtaskset$). In lines \ref{algline:whileactionpoints} to \ref{algline:endwhileactionpoints}, we iteratively increase the number of action points used until the problem returns sat or the sufficient number of action points is reached in which case if the problem still returns unsat it will continue to return unsat as is discussed in section \ref{sec:theoretical}. Lines \ref{algline:checkminactionpoints} to \ref{algline:endminactionpoints} increase the number of action points until the necessary number for the total number of tasks is reached. Lines \ref{algline:solve} through \ref{line:dp_increase2} check the encoding with the given assumption, returning a plan if sat and incrementing the index into the action point count list if not sat. Line \ref{algline:popoff} pops the most recent model.

\section{Theoretical Results} \label{sec:theoretical}
In this section, we prove soundness and completeness for Algorithm~\ref{alg:overall} as described in Def.~\ref{def:soundness} and \ref{def:completeness}.
We assume the solver used is sound and complete for the theories of bitvectors, uninterpreted functions, and linear integer arithmetic.
Due to space constraint, additional proofs have been moved to the appendix. 
Proofs for Lemmas $\ref{lemma:consistent_action_seq}-\ref{lemma:updated_completion}$ and Lemmas $\ref{lemma:dp_max}-\ref{lemma:incr_to_max}$ can be found in Appendix~$\ref{appendix:soundness}$ and $\ref{appendix:completeness}$, respectively.

\subsection{Soundness}

We first state two lemmas about our algorithm providing consistent action sequences and completing tasks given an initial task set of $\taskset_0$. To help in proving soundness, we also define strictly monotonically increasing action point times as $\phi = \forall \agent \in \agentset \ \forall \dpvar=1, \ldots, \numdp-1 \ \text{st.} \ \dpid \neq n, \prevdptime < \dptime$.

\begin{lemma}
    \textbf{Consistency of action sequences.} If $result_0$ is sat, $\phi \wedge \plan_0 \models \allconsistentpredicate$. \label{lemma:consistent_action_seq}
\end{lemma}

\begin{proof}
    (Sketch) We show only move/pick or move/drop pairs will be added to each action sequence if it is not empty. Therefore, the first action will be a move, there will be no consecutive moves, and corresponding moves will proceed picks/drops. Then, each second action can be mapped to the id for an action point, which constrains picks to be before drops and vice versa. We next show action point times are unique per agent by arguing that the alternative would require assigning two different times to the same task start or end time. This allows us to show picks and drops only occur once. Finally, we map loads of the action sequence to loads in the encoding to show capacity is adhered to.
\end{proof}

\begin{lemma}
    \textbf{Completion of tasks.} If $result_0$ is sat, $\plan_0 \models \allcompletedpredicate[\taskset]$ for $\taskset = \taskset_0$. \label{lemma:completed_tasks}
\end{lemma}

\begin{proof}
    (Sketch) We first show that if $result_0$ is sat, by construction each task should be started by an agent and be completed by the same agent. Moreover, no two agents can start or complete the same task because that would imply that two distinct values are assigned to $\taskagent$ in~\ref{enc:ended_task}. Finally, by construction the action point times satisfy the start time and deadline constraints, and by mapping time values to durations of subsequences we can prove that all timing constraints are satisfied.
\end{proof}

We next state three lemmas that combined state that a sat result for iteration $j+1$ means our algorithm will fulfill $\requirement$ provided a previous result of sat also led to the plan fulfilling $\requirement$.

\begin{lemma}
    \textbf{Plans are updated.} Assume $(result_j = sat) \Rightarrow (\phi \wedge \plan_j \models \requirement)$. If $result_{j+1}$ is sat, $\plan_{j+1} \models \indexedupdatedpredicate[t_{j+1}][\plan_j]$.
    \label{lemma:updated_plans}
\end{lemma}

\begin{proof}
    (Sketch) We are updating from a valid plan if $(result_j = sat)$. Consider two cases per agent: 1) The duration of the previous action sequence is less than $t_{j+1}$. SavePastState($\cdot$) will save action points that only have times in the past, so either the new action sequence will be the same or will include a wait. This is because the time between the last action point of the previous assignment and the first changed action point will be greater than just a travel time plus a pick/drop. 2) Duration is greater than or equal to $t_{j+1}$. The time of the last action point copied from the previous assignment will be $\geq t_{j+1}$, so there will be an equivalent prefix, and no new waits will be added.
\end{proof}

\begin{lemma}
    \textbf{Updated consistency.} Assume $(result_j = sat) \Rightarrow (\phi \wedge \plan_j \models \requirement)$. If $result_{j+1}$ is sat, $\plan_{j+1} \models \allconsistentpredicate$.
    \label{lemma:updated_consistency}
\end{lemma}

\begin{proof}
    (Sketch) Much of the proof can be repeated from that of Lemma \ref{lemma:consistent_action_seq}. The key differences are the potential for added wait actions. Therefore, simple pairs of a move and pick/drop will now occasionally be a wait, move, and pick/drop. The waits will shift the mapping between actions and action id number but the statements from before still hold i.e., that a task is picked before dropped and vice versa. Action point times are still strictly monotonically increasing. We show this by remembering that the saved action point times are strictly monotonically increasing. We know that the new ones will be strictly monotonically increasing and greater than $t_{j+1}$, so the whole sequence will be, and we can claim pick and drop actions only occur once as before.
\end{proof}

\begin{lemma}
    \textbf{Updated completion.} Assume $(result_j = sat) \Rightarrow (\phi \wedge \plan_j \models \requirement)$. If $result_{j+1}$ is sat, $\plan_{j+1} \models \allcompletedpredicate[j+1]$ where $\cumtaskset_{j+1} =(\bigcup_{j'=0}^{j+1} \taskset_{j'})$. \label{lemma:updated_completion} 
\end{lemma}

\begin{proof}
    (Sketch) Much of the proof can be repeated from that of Lemma \ref{lemma:completed_tasks}. For tasks that are completed before the arrival time $\tasksetarrivaltime$ of new tasks, the duration constraints are satisfied. For all action points that take place after $\tasksetarrivaltime$, since the encoding still constrains that each task $\task$ starts after $\taskstarttime$ and ends before $\taskendtime$ and we are able to map action point times to durations, the timing constraints will be fulfilled.
\end{proof}

In Theorem $\ref{theorem:soundness}$ and the following proof, we use the above lemmas to build an inductive argument that our algorithm is sound.

\begin{theorem}
    \textbf{Soundness of Algorithm.} Algorithm \ref{alg:overall} is sound. \label{theorem:soundness}
\end{theorem}

\begin{proof}
     Assuming $result_0 = sat$, the initial plan $\plan_0 \models \updatedpredicate[\emptyset]$ by definition. By Lemmas \ref{lemma:consistent_action_seq} and \ref{lemma:completed_tasks}, $\phi \wedge \plan_0 \models \allconsistentpredicate \wedge \allcompletedpredicate$ where $\indexedcumtaskset = \taskset_0$. We have therefore shown a base case that $(result_0 = sat) \Rightarrow (\phi \wedge \plan_0 \models \requirement[0])$.
     
     We will next make an inductive argument to show that at each iteration we are building a valid, updated plan with strictly monotonically increasing action point times. We need to show for $j \geq 1$ that $((result_{j-1} = sat) \Rightarrow (\plan_{j-1} \models \requirement[j-1])) \Rightarrow (result_{j} = sat \Rightarrow \plan_{j} \models \requirement)$. We have two cases: 1) $result_{j-1} = sat$ and 2) $result_{j-1} \neq sat$. By construction, the algorithm will end and return unsat if $result_{j-1} \neq sat$ for $j>0$, so if $result_{j-1} \neq sat$ then $result_{j} \neq sat$ for $j>0$ then each implication is true proving the statement for this case. For $j>0$, if we assume $(result_{j-1} = sat \Rightarrow \plan_{j-1}\models \requirement[j-1]) \wedge\ (result_{j-1}$ = sat), then by Lemmas \ref{lemma:updated_plans}, \ref{lemma:updated_consistency}, and \ref{lemma:updated_completion}, $result_{j} = sat$ implies $\phi \wedge \plan_j \models \requirement$ and $result_{j} \neq sat$ trivially satisfies the formula. The actual statement of soundness $(result_j = sat) \Rightarrow (\plan_j \models \requirement[j])$ can be implied from this result.
\end{proof}

\subsection{Completeness}

In the following three lemmas, we state first that $D_{max}$ is the maximum number of action points needed for an encoding provided the max number of tasks that will arrive is known. Then, we show that given we start with this maximum number of action points our algorithm is complete. Finally, we state that we will reach this number of action points in our algorithm if necessary. This allows us to then prove completeness.

\begin{lemma}
    \textbf{Maximum number of action points.} If an assignment $\modeloutput$ does not exist for an encoding $\encoding$ when $D = D_{max} = 2\numtasks + 1$ of action points then no assignment $\modeloutput$ exists for $D > D_{max}$. \label{lemma:dp_max}
\end{lemma}

\begin{proof}
    As shown in Lemma \ref{lemma:consistent_action_seq}, pick and drop ids will only occur once per task in an agent's assignment in a satisfying $\modeloutput$. By construction, action point ids must be either $\dpid = \agent$ or $\dpid \geq \numagents \wedge \dpid < 2\numtasks + \numagents$. Therefore, for an agent $\agent$, the maximum number of action points that can be assigned to a value other than $\agent$ is $2\numtasks$. Adding in the constrained 0th action point, the total is $2\numtasks + 1 = D_{max}$ where $\dpid = n$ for $d = d' \geq D_{max}$. Therefore, adding an extra action point does not add new free variables, so an unsat result cannot turn sat by adding more action points.
\end{proof}

\begin{lemma}
    \label{lemma:complete_given_max}
    \textbf{Conditional Completeness.} Assume $\Gamma[|\dplist| -1]$. If $D=D_{max}$, Algorithm~\ref{alg:overall} is complete.
\end{lemma}

\begin{proof}
    (Sketch) We want to show that for each iteration $j$ in Algorithm 1, if given a plan $\plan_{j}$ for which $\plan_{j} \models \requirement$, we can find a satisfying assignment $\modeloutput_j$. Take the action sequence for each agent $\agent$. Set the initial action point $\dptuple[\agent][0] = (\agentid, \initload, \inittime)$. Find the 1-indexed subsequence of indices of pick and drop actions. For each element $k$ in the subsequence, let $\dpid[\agent][k] = 2\taskid + \numagents$ or $\dpid[\agent][k] = 2\taskid + \numagents + 1$ where $\taskid$ is the id of the task that is picked or dropped, respectively. Set $\dptime[\agent][k] = \duration_k(\actionsequence)$ and $\dpload[\agent] = \sequenceload_k(\actionsequence)$. For the task $\task$, if $\sequenceelement \in \pickset$, set $\taskstarttime = \duration_k(\actionsequence)$. If $\sequenceelement \in \dropset$, set $\taskendtime = \duration_k(\actionsequence)$. Set $\taskagent = \agentid$. Set all remaining action points to $(\agentid, t_{max}, 0)$. In Appendix \ref{appendix:completeness} we show that this assignment is satisfying by going through constraints in the encoding.
\end{proof}

\begin{lemma}
    \label{lemma:incr_to_max}
    \textbf{Increment to $\boldsymbol{D_{max}}$.} Following Algorithm \ref{alg:overall}, k will eventually be such that the number of free action points $D = D_{max}$.
\end{lemma}

\begin{proof}
    An action point $d$ is free if it is not constrained to have $\dpid = \agentid$.  By inspection we see that the while loop in Lines \ref{line:dp_increase1} and \ref{line:dp_increase2} increases the index into the action point list which changes the assumes until the last one which then places no restriction on the encoding. By construction, the encoding can use all $D_{max}$ action points when no assumptions are present.
\end{proof}

\begin{theorem}
    \textbf{Completeness of Solve.} Algorithm \ref{alg:overall} is complete. \label{theorem:completeness}
\end{theorem}

\begin{proof}
    By Lemma \ref{lemma:incr_to_max}, we will reach $\numdp = D_{max}$. By Lemma \ref{lemma:complete_given_max}, we know that given a sufficiently large $\numdp = D_{max}$, the algorithm is complete.
\end{proof}

\section{Experimental Analysis}

In this section, we evaluate the performance of our approach for initial solves and incremental solves for tasks arriving online. Specifically, we aim to answer the following research questions:

\begin{enumerate} [label={}]
    \item \textbf{RQ1}: How does the initial solve scale with respect to the number of tasks, number of agents, and number of action points? How does it scale with respect to different theory encodings?
    \item \textbf{RQ2}: How does incremental solving scale, and is it more effective than non-incremental solving at handling tasks arriving online? How does task set batch size affect the performance?
\end{enumerate}

For the initial solve, we generated benchmarks using the Z3Py API for initial solve in both QF\_UFLIA and QF\_UFBV.
For conciseness, we abbreviate these as LIA and BV, respectively. For incremental solving, we implement the encoding in BV using both Z3Py and the Bitwuzla~\cite{DBLP:conf/cav/NiemetzP23} Python API. For BV benchmarks, variables are encoded using small-domain encoding~\footnote{Link to implementation and benchmarks: \href{https://github.com/victoria-tuck/SMrTa}{https://github.com/victoria-tuck/SMrTa}}. 


Our workspace model represents an indoor setting with hallways and twenty rooms. It aims to reflect the complexities and challenges robots face in navigating complicated indoor settings. System locations were chosen from critical regions and travel time between pairs of locations were calculated with the approach in \cite{ShahS2022aamas}. Agents' start locations and tasks start and end locations are uniformly randomly sampled from the system locations. 
Our evaluation consists of two sets of benchmarks, discussed separately in Section~\ref{subsec:rq1} and \ref{subsec:rq2}. We use Z3, Bitwuzla, and cvc5 as state-of-the-art SMT solvers. 
In the following, we abbreviate $S$-$T$ as a solver setting where solver $S$ runs on some benchmarks in theory $T$.
All experiments were run on an AWS EC2 c5.4xlarge instance with 16 Intel Xeon Platinum cores running at 3.0GHz with 16GB RAM. All solvers were given a 3600 seconds timeout on each query.

\subsection{RQ1: Performance of Initial Solve and Comparison on Theories} \label{subsec:rq1}

\begin{figure}
    \centering
    \includegraphics[width=0.7\linewidth]{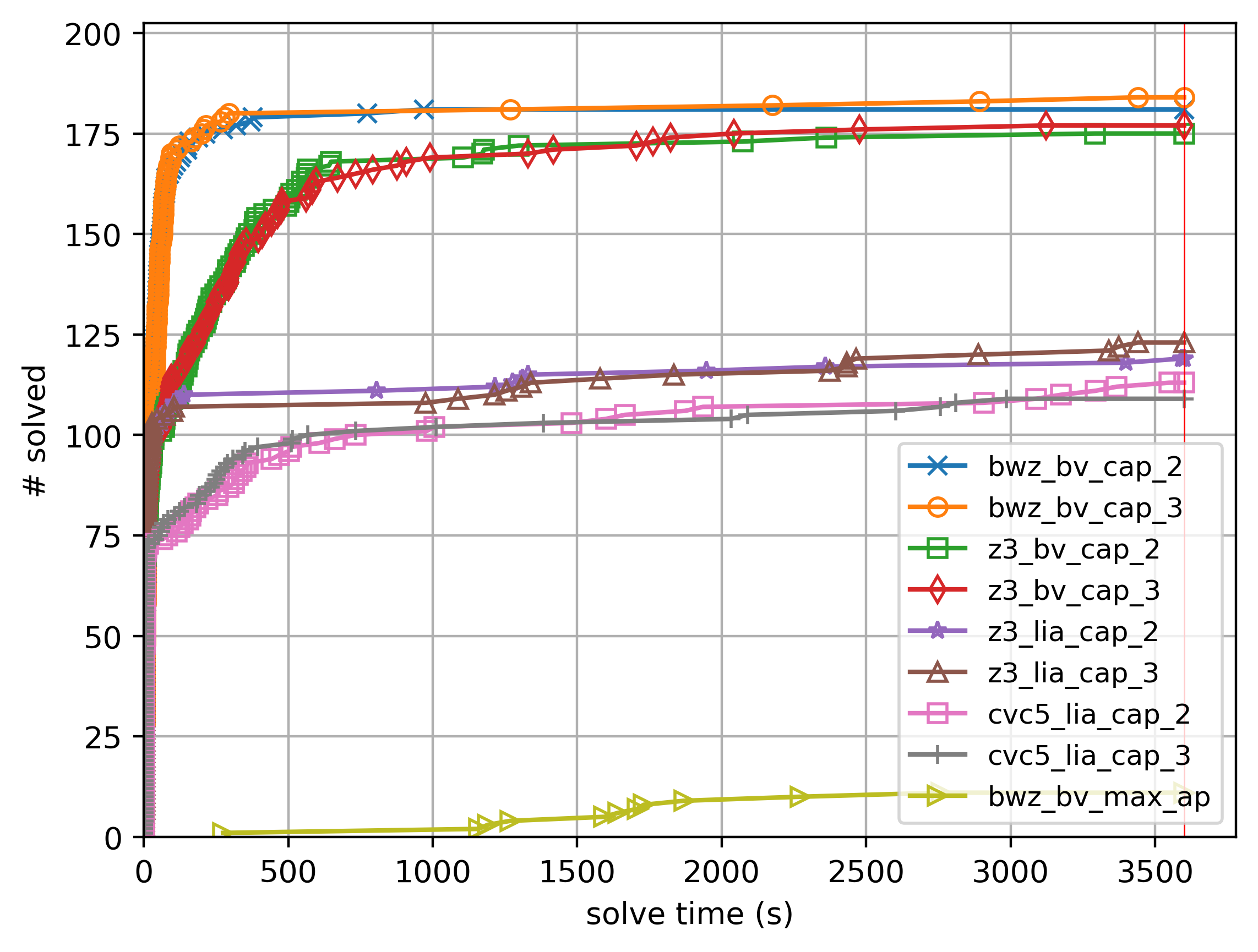}
    \caption{Cactus plot for solver runtimes. Red line represents the 3600s timeout.}
    \label{fig:cactus}
\end{figure}

We create a set of 200 benchmarks with number of tasks ranging from 10 to 30 and number of agents from 5 to 20 both with an interval of 5 and ten instances for each combination. Task deadlines are sampled from a uniform distribution. Minimum number of action points were used. We ran both Z3-BV and Z3-LIA, Bitwuzla-BV, and cvc5-LIA on the initial solve benchmarks with max capacity $c$ equal to 2 or 3. 

Figure~\ref{fig:cactus} shows the cactus plot for each solver. Solvers running on BV benchmarks consistently outperformed those running on LIA. Bitwuzla-BV performed the best, solving 184 and 181 instances, followed by Z3-BV solving 177 and 175 with $c = 3$ and $2$ respectively. Solvers in general seem to perform better with larger max capacity, and we speculatively believe the constraint of $c = 3$ to be easier to solve. 
We also conducted an experiment using the maximum number of action points (\texttt{bwz\_bv\_max\_ap} in Fig.~\ref{fig:cactus}) to show that action point number significantly affects performance.

Figure~\ref{fig:bwz-vs-z3} shows the relation between problem size and solver performance for Bitwuzla-BV and Z3-BV with $c = 3$. In both plots, runtime grows as number of tasks grows, and when the number of tasks is fixed, a larger number of agents, which implies a smaller minimum number of action points, tend to lead to faster runtimes. Bitwuzla-BV generally outperforms Z3-BV on benchmarks with $\numtasks > 25$, while Z3-BV outperforms Bitwuzla-BV on those with $\numtasks <15$.

\begin{figure}
    \centering
    \includegraphics[width=\linewidth]{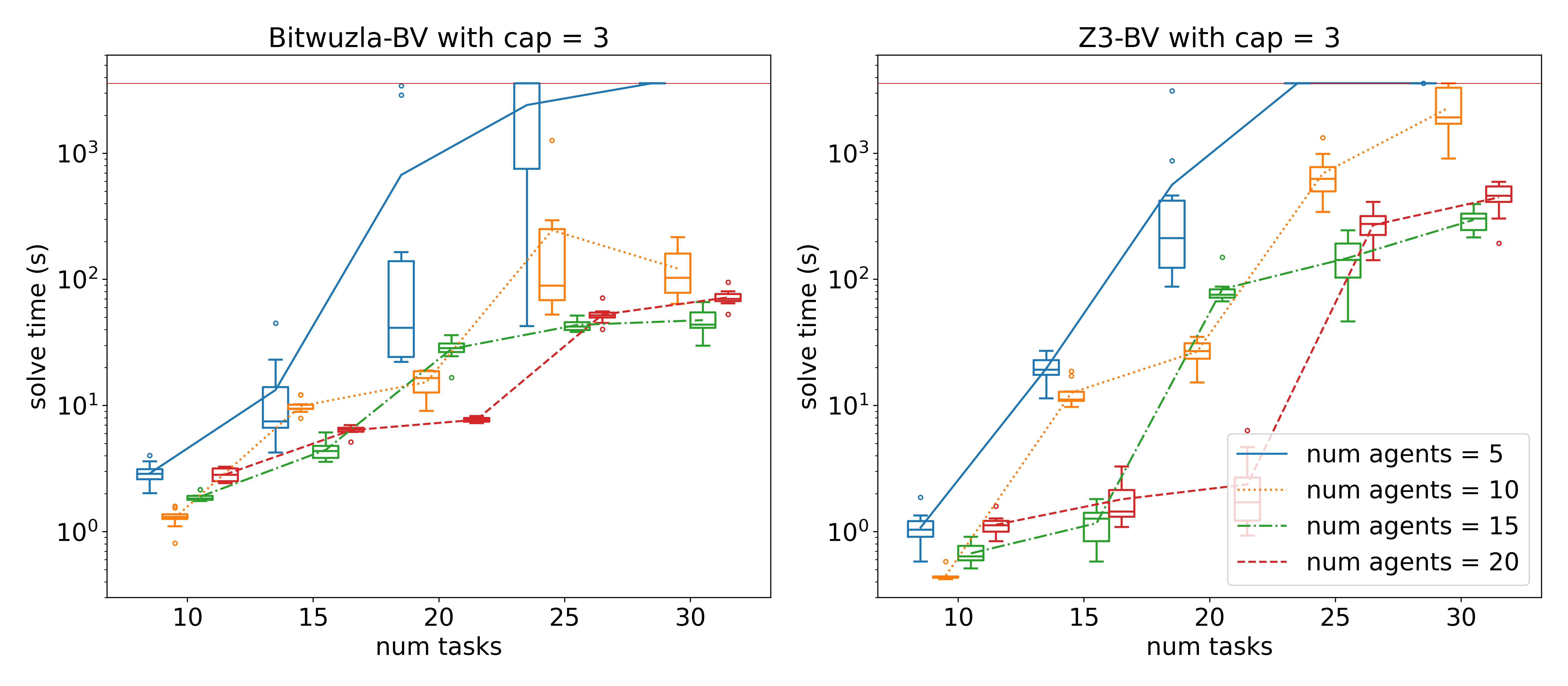}
    \caption{Runtime analysis on Bitwuzla-BV and Z3-BV under multiple settings. Boxes represent quantiles, circles represent outliers, and lines represent means of runtimes. Observe that runtimes are faster when number of agents is larger.}
    \label{fig:bwz-vs-z3}
\end{figure}

\subsection{RQ2: Performance of Incremental Solve} \label{subsec:rq2}

We create a set of 20 benchmarks with 20 agents that simulate a real world scenario where tasks arrive continuously. With a total of 200 tasks, we assume that each task arrives every 8 time units, and expires in $t \sim U(300, 500)$ time units after its arrival.
We use batch (task set) sizes $b \in \{1, 10\}$. For each batch size $b$, the algorithm collects tasks and invokes the solver every $b$ tasks. Due to the superior performances in the initial solve, only Bitwuzla-BV and Z3-BV with $c = 2$ were considered. Tasks are added incrementally via push/pop functionality, and action points are added using assumption variables according to the approach shown in Algorithm \ref{alg:overall}. For non-incremental solving, we copy all assertions of the incremental solver without the pushes/pops to a newly instantiated solver and assert the assumption variables.

We timed the execution of both solvers using incremental solving and non-incremental solving. The number of total free action points across agents was also recorded as an indicator of query difficulty. 
Fig.~\ref{fig:inc-vs-noninc} shows the results of running Bitwuzla-BV and Z3-BV with batch size equal to 1 and 10 with 200 tasks and 20 agents. 
We observed that performances on incremental and non-incremental solving depend greatly on the solver and batch size -- Z3-BV significantly outperforms Bitwuzla-BV on incremental solving especially when batch sizes are small, as shown in the blue line in Fig.~\ref{fig:z3_bv_b_1}, while Bitwuzla-BV performs better on non-incremental solves with larger batch sizes, as shown in the orange line in Fig.~\ref{fig:bwz_bv_b_10}. Empirically, Z3-BV took around 260 seconds on average to solve for 200 tasks. Based on the observations, we suggest using Z3-BV/Bitwuzla-BV with incremental/non-incremental solving when batch sizes are small/large.

Notice that there are peaks in runtime across all settings. These peaks occur when the minimum number of action points required increases. With 20 agents and batch size equal to 1 (10), a peak occurs every 20 (2) batches as an additional action point is needed every 20 tasks. We speculate this to be due to an increase of the search space when extra action points are introduced (shown via red lines in Fig.~\ref{fig:inc-vs-noninc}). Note that runtime does not correlate to the number of action points that have to be assigned to complete all available tasks, which is an indicator of the number of un/re-assigned tasks (shown in green in Fig.~\ref{fig:inc-vs-noninc}).

In comparison to heuristic-based approaches~\cite{sarkar2018scalable}, those approaches will be faster but lack the guarantees of our approach. Additionally, when comparing others with guarantees~\cite{okubo2022simultaneous,chen2021integrated}, these lack runtime information. 

\begin{figure}
    \centering
    \begin{subfigure}[b]{0.45\textwidth}
        \centering
        \includegraphics[width=\textwidth]{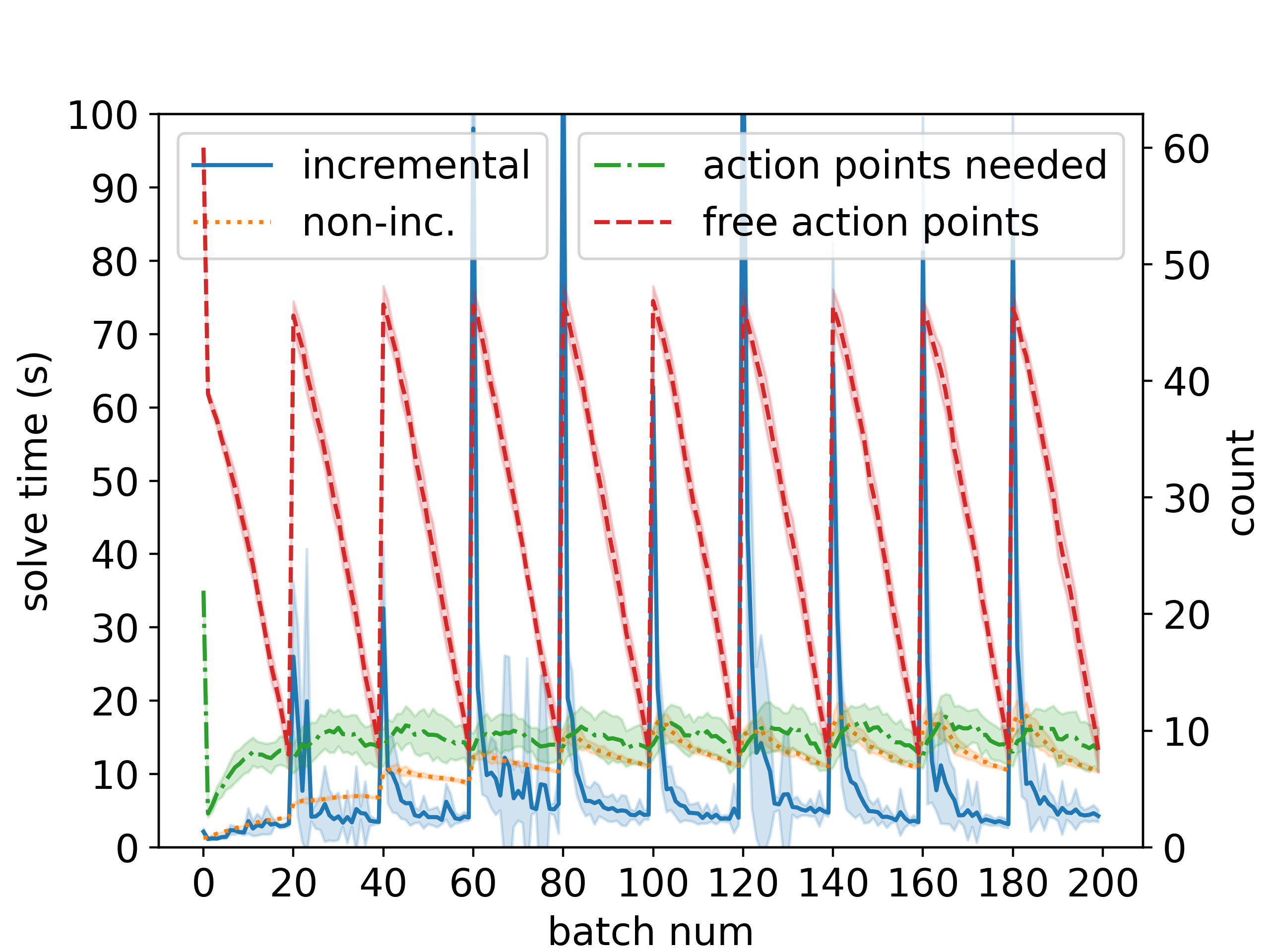}
        \caption{Bitwuzla-BV with batch size = 1.}
        \label{fig:bwz_bv_b_1}
    \end{subfigure}
    \begin{subfigure}[b]{0.45\textwidth}
         \centering
         \hfill 
         \includegraphics[width=\textwidth]{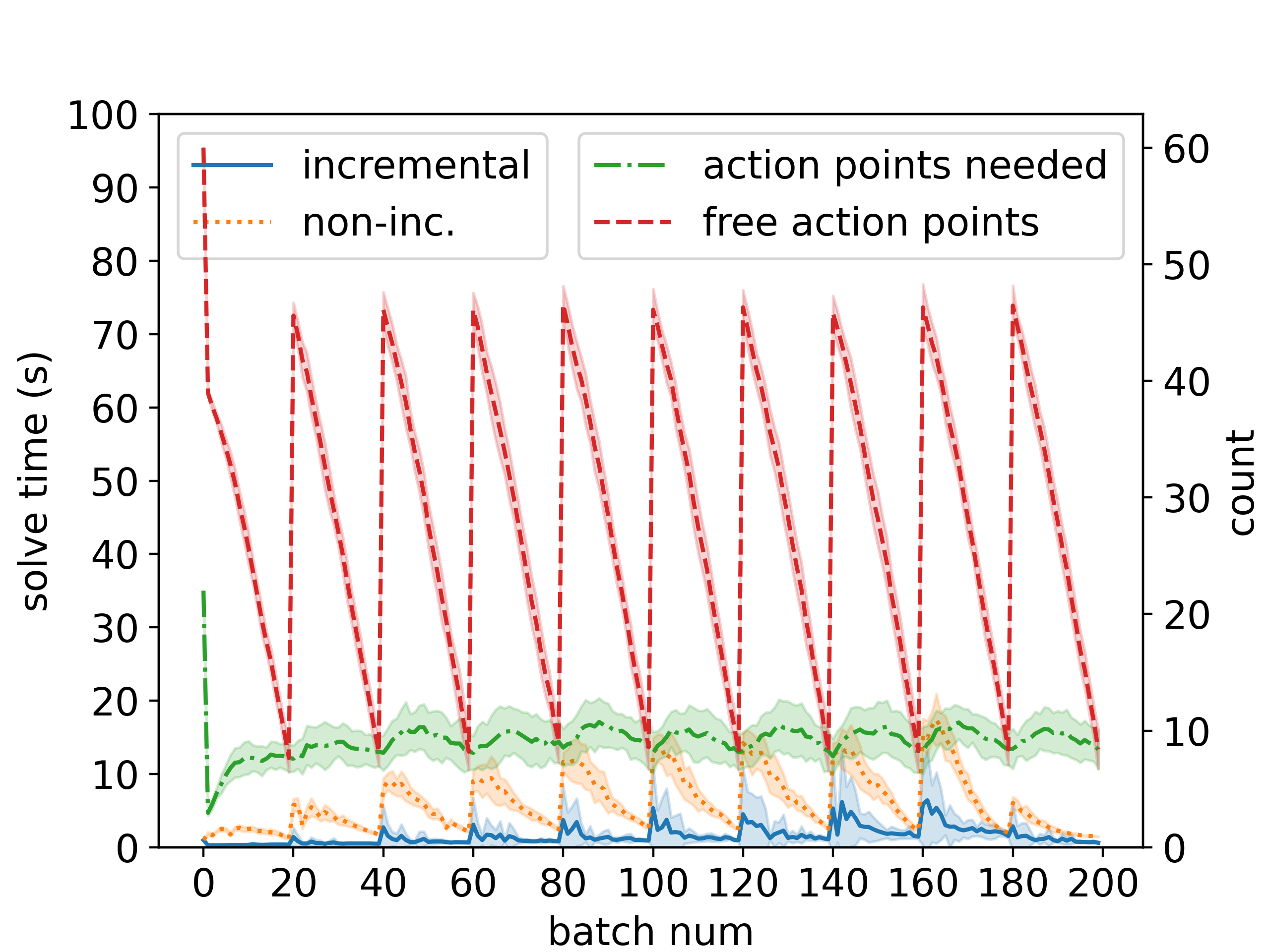}
         \caption{Z3-BV with batch size = 1.}
         \label{fig:z3_bv_b_1}
     \end{subfigure}
     \begin{subfigure}[b]{0.45\textwidth}
        \centering
        \includegraphics[width=\textwidth]{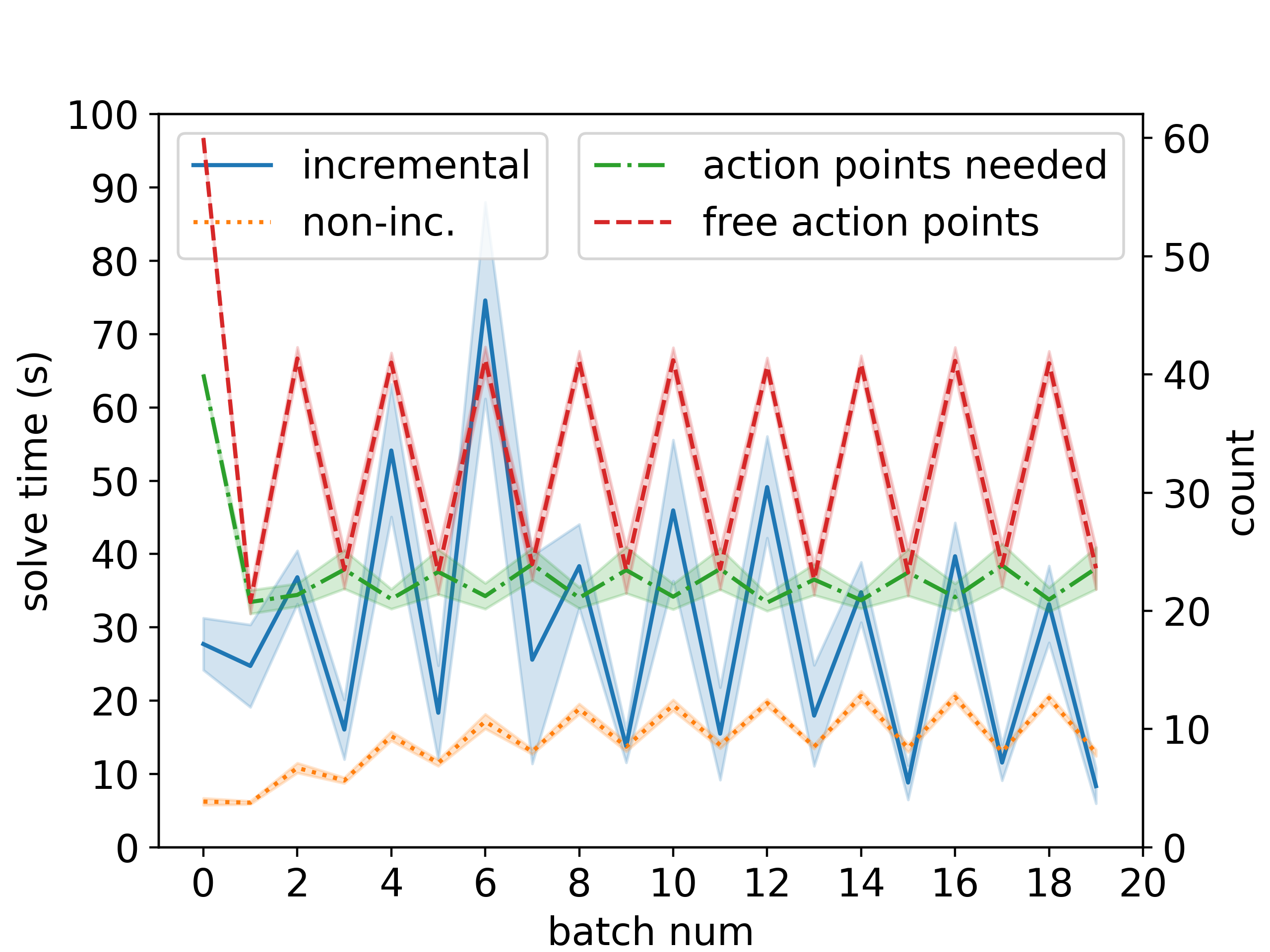}
        \caption{Bitwuzla-BV with batch size = 10.}
        \label{fig:bwz_bv_b_10}
    \end{subfigure}
    \begin{subfigure}[b]{0.45\textwidth}
         \centering
         \hfill 
         \includegraphics[width=\textwidth]{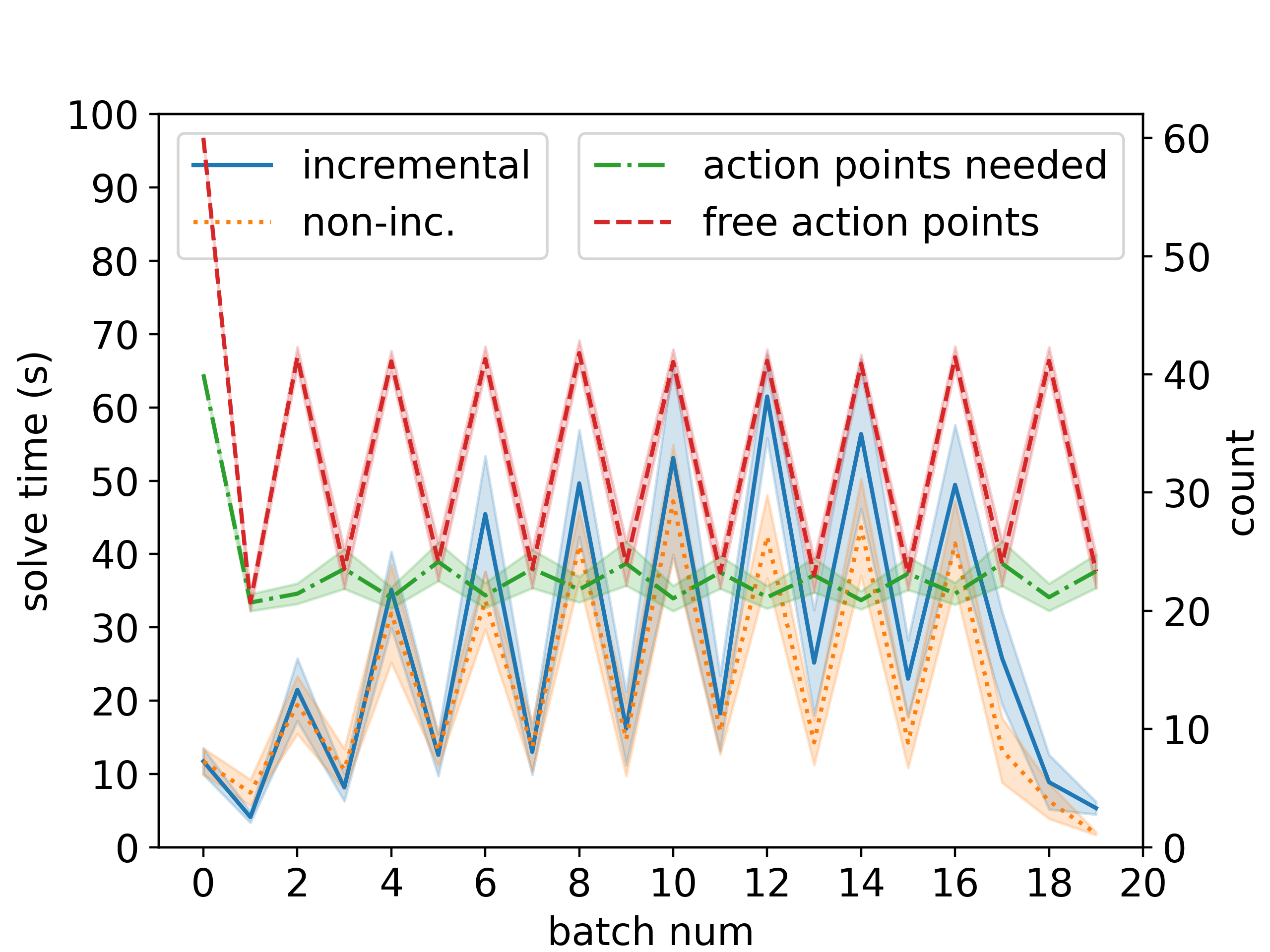}
         \caption{Z3-BV with batch size = 10.}
         \label{fig:z3_bv_b_10}
     \end{subfigure}
    \caption{Performance comparison for incremental v.s. non-incremental solves with $cap = 2$.}
    \label{fig:inc-vs-noninc}
\end{figure}

\section{Conclusion and Future Work}

In this work, we present a SMT-based approach to the problem of Dynamic Multi-Robot Task Allocation with capacitated agents. Our algorithm handles online tasks and iteratively adjusts the size of the problem in order to manage computational complexity. We show its efficacy on problems of up to 20 agents and 200 tasks, showing the potential for our approach to be used in longer settings. Future work includes extending to stochastic settings to better accommodate dynamic environments where a probabilistic guarantee of adherence is desired potentially using probabilistic logics \cite{hansson1995}. Further extensions include allowing agents to pass objects to one another, and, specifically for hospital settings, representing complicated features like elevators. We also believe that this problem setting can act as an SMT benchmark because, for example, without managing action points as in our approach, large numbers of action points do create problems that are very difficult to solve as shown in Fig. \ref{fig:cactus}. In conclusion, we show SMT-based approaches can be useful in handling the computational complexity of this combinatorial domain in an extendable way and that our framework provides a baseline for further study of this area.

\subsubsection*{Acknowledgements.}

This work was supported in part by LOGiCS: Learning-Driven Oracle-Guided Compositional Symbiotic Design for Cyber-Physical Systems, Defense Advanced Research Projects Agency award number FA8750-20-C-0156; by Provably Correct Design of Adaptive Hybrid Neuro-Symbolic Cyber Physical Systems, Defense Advanced Research Projects Agency award number FA8750-23-C-0080; by Toyota under the iCyPhy Center; and by Berkeley Deep Drive.

%
\bibliographystyle{abbrv}
\bibliography{refs}

\appendix

\section{Soundness} \label{appendix:soundness}

\subsection{Proof of Lemma \ref{lemma:consistent_action_seq}: Consistency of Action Sequences.} \label{appendix:lemma_1}

Lemma statement: If $result_0$ is sat, $\plan_0 \models \allconsistentpredicate$.

\begin{proof}
$(result_0 = sat) \Rightarrow \modeloutput_0$ is a satisfying assignment to the encoding $\encoding_0$. The construction of $\plan_0$ is deterministic, so we will only get one possible $\plan_0$ from an assignment.

Consider any agent's action sequence $\actionsequence \in \plan_0$. If an agent's action sequence $\actionsequence$ is empty, $\actionsequence \models \consistentpredicate$. If it is not, let its length be $\sequenceindex$. By construction, $((\sequenceelement = \pick_{\task}) \Rightarrow (\prevsequenceelement = \move_{\taskstartloc})) \wedge ((\sequenceelement = \drop_{\task}) \Rightarrow (\prevsequenceelement = \move_{\taskendloc})) \ \forall \task \in \taskset$ and $(\sequenceelement[1] \in \moveset)$.

Taking each action point $\dpvar > 0$, if $\validid$ is true, the point will cause the action pair $\{(\move, \taskstartloc), (\pick, \taskid) \}$ or $ \{(\move, \taskendloc), (\drop, \taskid) \}$ to be added to the action sequence $\actionsequence$. Otherwise, due to \ref{enc:at_home_or_doing_task_with_id}, $\dpid=\agentid$. In constructing $\actionsequence$ following GetPlan($\modeloutput_0$), we start with the 1st action point and stop adding elements to the $\actionsequence$ once $\dpid = \agentid$. Therefore, only the pick and drop action pairs can be added, so $\forall \kappa = 1, \ldots, k-1 \ \sequenceelement[\kappa] \in \moveset \Rightarrow \sequenceelement[\kappa + 1] \in (\pickset \cup \dropset)$.  $(\pickset \cup \dropset) \cap \moveset = \emptyset$, so $\forall \kappa = 1, \ldots, k-1 \ \sequenceelement[\kappa] \in \moveset \Rightarrow \sequenceelement[\kappa + 1] \notin \moveset)$ .

Because we only add the pairs mentioned earlier to the action sequence, for any action point $\dpvar$ and task $\task$, if $\dpid = 2\taskid + \numagents$, then $\sequenceelement[2\dpvar] = (\pick, \taskid)$. Due to \ref{enc:started_task}, $\dpid[\agent][\dpvar'] = 2\taskid + \numagents + 1$ for some $\dpvar' > \dpvar$. $\dpid = 2\taskid + \numagents + 1$ implies $\sequenceelement[2\dpvar'] = (\drop, \taskid)$. $\dpvar' > \dpvar \Rightarrow 2\dpvar' > 2\dpvar$, so tasks will be picked before they are dropped ($(\sequenceelement[\kappa] = \pick_{\task}) \Rightarrow (\bigvee \limits_{\kappa' > \kappa} \sequenceelement[\kappa'] = \drop_\agent)$ holds where $\kappa = 2\dpvar$ and $\kappa' = 2\dpvar'$). A similar argument can be made with \ref{enc:started_task} - \ref{enc:assigned_agent} to show tasks are dropped after they are picked.

We now show that the action point times $\dptime$ across the same agent that will be used to create the agent's action sequence are unique in order to show that picks and drops only occur once. More formally, there exists some $0 < \numdp_{\agent}' \leq \numdp$ where 
\begin{align}
    \forall d = 1, \ldots, \numdp_{\agent}'-1, \quad & \dpid \neq \agentid \label{eq:below_Dprime_neq_n} \\
    \forall d' = 1, \ldots, \numdp_{\agent}'-1 \text{ st. } d \neq d', \quad & \dptime \neq \dptime[\agent][d'] \label{eq:below_Dprime_uniquetimes}
\end{align}
and 
\begin{equation}
    \forall \ d'' \text{ st. } \numdp_{\agent}' \leq d'' < D, \quad \dpid[\agent][d'']=\agentid. \label{eq:above_Dprime_equaln}
\end{equation} 
From \ref{enc:room_UF_agents_and_dist_UF}, \ref{enc:at_home_or_doing_task_with_id}, \ref{enc:room_UF_task_rooms} and $w_{\locationid_1, \locationid_2} \geq 0 \ \forall, \locationid_1, \locationid_2 \in \locationidset$, $Dist(Loc(\prevdpid), Loc(\dpid)) \geq 0$. Now we can find $\numdp_{\agent}'$. Start with $d=1$. If $\dpid=\agentid$, via \ref{enc:stay_home}, eq. \ref{eq:below_Dprime_neq_n}, \ref{eq:below_Dprime_uniquetimes}, and \ref{eq:above_Dprime_equaln} trivially hold with $\numdp_{\agent}'=1$. For all $d>1$, if $\dpid \neq \agent$, $\dpid \geq \numagents$ by \ref{enc:at_home_or_doing_task_with_id}. In \ref{enc:doing_tasks}, $ITE(\prevdptime \leq \taskarrivaltime, \taskarrivaltime, \prevdptime) \geq \prevdptime$. By definition, $\rho > 0$, so by \ref{enc:doing_tasks}, $\dptime > \prevdptime$. This sequence of times is strictly monotonically increasing, so all values are unique. If for any action point $d>1$, $\dpid=\agent$, let $\numdp_{\agent}'=d$ and by \ref{enc:stay_home}, eq. \ref{eq:above_Dprime_equaln} holds. Eq. \ref{eq:below_Dprime_neq_n} and \ref{eq:below_Dprime_uniquetimes} hold by the above argument of strictly monotonically increasing sequence. If for no action point $d>1$, $\dpid=\agent$, let $\numdp_{\agent}'=\numdp$ and eq. \ref{eq:above_Dprime_equaln} trivially holds while eq. \ref{eq:below_Dprime_neq_n} and \ref{eq:below_Dprime_uniquetimes} hold by the above argument.

We can now show 
\begin{align*}
    (\sequenceelement[\kappa] = \pick_{\task}) \Rightarrow & (\forall \kappa' \neq \kappa \in 1, \ldots  \sequenceindex, \ \sequenceelement[\kappa'] \neq \pick_{\task}) \ \forall \task \in \taskset \\
    (\sequenceelement[\kappa] = \drop_{\task}) \Rightarrow & (\forall \kappa' \neq \kappa \in 1, \ldots  \sequenceindex, \ \sequenceelement[\kappa'] \neq \drop_{\task}) \ \forall \task \in \taskset
\end{align*}
for each agent $\agent$. We will only consider the action points for each agent that are used to create its action sequence ($d=0, \ldots, \numdp_{\agent}'-1$). We have shown that these times $\dptime$ are unique. Due to \ref{enc:started_task}, $\dpid = 2 \taskid + \numagents \Rightarrow \taskstarttime = \dptime$ (if an action point is assigned a pick up task id, $\taskstarttime = \dptime$ will be assigned the time associated with that id). If two action points $\dpvar$ and $\dpvar'$ for one agent are assigned the same task id, $\taskstarttime$ must be assigned both $\dptime$ and $\dptime[\agent][\dpvar']$. However, this is impossible as $\dptime \neq \dptime[\agent][\dpvar']$. Therefore, only one action point for each agent can be assigned the same pick up task id. A similar argument can be made using \ref{enc:ended_task} for drop off task ids. When converted to an action sequence, only one pick or drop action is added to the action sequence for each action point, so the pick and drop actions will only occur once.

We now show that loads will not violate specifications. If the action sequence is non-empty, by construction, $\seqcapacity_{1}(\actionsequence) = 0$ and $\seqcapacity_{k} (\actionsequence) = \dpload$ where $d = \lfloor \frac{k}{2}\rfloor$ due to \ref{enc:load_tracking_and_doing_tasks}. These load values are constrained in the encoding by $\ref{enc:bound_load_bound_id_go_home}$. The constraints can be similarly shown for all other agents.

\end{proof}

\subsection{Proof of Lemma \ref{lemma:completed_tasks}: Completion of Tasks.} \label{appendix:lemma_2}

Lemma statement: If $result_0$ is sat, $\plan_0 \models \allcompletedpredicate[\taskset]$ for $\taskset = \taskset_0$. 

\begin{proof}
    We need to show $ \plan_0 \models \completepredicate \forall \task \in \taskset_0$. Without loss of generality, we will show for a single task $\task \in \taskset_0$ that $\exists \ \actionsequence \in \plan$ such that each statement in $\completepredicate$ holds.
    
    The action sequence is consistent ($\actionsequence \models \consistentpredicate)$ by Lemma \ref{lemma:consistent_action_seq}.

    We show the object for task $\task$ is picked up and dropped off by agent $\agent \ (\phi_{pickup}(\task, \actionsequence) \wedge \phi_{dropoff}(\task, \actionsequence))$: By \ref{enc:valid_agent_and_times}, for the task $\task$, $(\taskagent \geq 0) \wedge (\taskagent < \numagents)$. WLOG, assume $\taskagent = \agent$. By $\ref{enc:assigned_agent}, \exists \ \dpvar \in [1, \numdp - 1] \ \text{st.} \ \dpid = 2\task + \numagents.$ Therefore, via GetPlan($\modeloutput$), $\pick_{\task}$ is added to $\actionsequence$. Similarly, using $\ref{enc:started_task}$, $\drop_{\task}$ is added with an action point $\dpvar' = 2 \task + N + 1$.

    We show no other agent will pick up the task ($\forall \otheractionsequence \neq \actionsequence \in \plan, \neg \phi_{pickup}(\task,$ $ \otheractionsequence) \wedge \neg \phi_{dropoff}(\task, \otheractionsequence)$): Via $\ref{enc:started_task}$ and $\ref{enc:ended_task}$, $\dpid[\agent'][\dpvar] = 2 \agent + N$ or $\dpid[\agent'][\dpvar] = 2 \agent + N + 1$ will imply $\taskagent = \agent'$. It can only be assigned one value. Therefore, neither the pick-up or drop-off of task $\task$ will be assigned to another agent $\agent'$ or appear in that agent's action sequence.
    
    We show valid durations ($\duration(\sequenceelement[\kappa_{d}]) \leq \taskdeadline) \wedge (\duration(\sequenceelement[\kappa_{p-2}]) \geq \taskarrivaltime)$: Via $\ref{enc:stay_home}, \prevdpid \neq \agent$ unless $d-1=0$. We previously assume $\tasksetarrivaltime = 0.$ Therefore,
    \begin{align*}
        \dptime - \prevdptime = & \ Dist(Loc(\prevdpid), Loc(\dpid)) + \rho \quad [\text{Via } \ref{enc:doing_tasks}] \\
        = & \ w_{Loc(\prevdpid), Loc(\dpid)} + \rho \quad [\text{Via } \ref{enc:room_UF_agents_and_dist_UF}] \\
        = & \ \duration_k(\actionsequence) - \duration_{k-2}(\actionsequence) \quad [\text{Via } \ref{enc:room_UF_agents_and_dist_UF}, \ref{enc:room_UF_task_rooms}]
    \end{align*}
    with $k = 2d$. Via \ref{enc:init_and_allowed_dp}, $\dptime[\agent][0] = 0$, so we can map durations of the subsequences up to and including pick and drop actions to action point times. Times are constrained in $\ref{enc:valid_agent_and_times}$. As a note, $\duration_{k-1}(\actionsequence) = \dptime[\agent][\lceil \frac{k}{2} \rceil] - \rho$, which is constrained in $\ref{enc:valid_agent_and_times}$ but we also want to show the stronger point that the agent does not even move towards the task until it knows about it. Duration is a sum of non-negative quantities and $\taskarrivaltime = 0$, so the second duration constraint holds.

\end{proof}

\subsection{Proof of Lemma \ref{lemma:updated_plans}: Plans are Updated.} \label{appendix:lemma_3}

Lemma statement: Assume $(result_j = sat) \Rightarrow (\plan_j \models \requirement)$. If $result_{j+1}$ is sat, $\plan_{j+1} \models \indexedupdatedpredicate[t_{j+1}][\plan_j]$.

\begin{proof}
    If $result_j$ is not sat, the statement is trivially true as the loop will exit on line $\ref{algline:exitunsat}$ before $result_{j+1}$ can be set to sat.

    For the rest of the proof, we assume $result_j$ is sat. Therefore, the right-hand side of the implication is true, so we are updating from a valid plan $\plan_{j}$. Consider the action sequences of the plan $\plan_{j}$. For each agent, we split into two cases: 1) $\duration_k(\actionsequence) < t_{j+1}$ and 2) $\duration_k(\actionsequence) \geq t_{j+1}$ where $\sequenceindex$ is the length of $\actionsequence$.

    Start with case 1. We note that via the procedure in GetPlan($\cdot$), $\duration_{\kappa}(\actionsequence)$ will equal an action tuple time for all $\kappa$'s where $\sequenceelement \in \pickset \cup \dropset$. Therefore, $\duration_k(\actionsequence) < t_{j+1}$ means that $\modeloutput_j(\dptime) \geq t_{j+1}$ will never be true and SavePastState($\cdot$) will stop at some $\delta_{\agent} = \dpvar + 1$ when $\dpvar+1 < \numdp$ and $\dpid[\agent][\dpvar + 1] = \agentid$. Therefore, all action points $\dpvar = 0, \ldots, \delta_{\agent}$ will be equivalent between the two assignments creating a prefix of the new action sequence that is equivalent to the previous action sequence. At this point, if there are no new action tuple times for the new assignment $\modeloutput_{j+1}$ i.e., $\dpid[\agent][\delta_{\agent}] = \agentid$, the statement is true. If $\dpid[\agent][\delta_{\agent}] \neq \agentid$, we note that $\dptime[\agent][\delta-1] < t_{j+1}$. Via $\ref{enc:doing_tasks}$ and lines $\ref{algline:popoff}$ and $\ref{algline:addtasks}$ in Algorithm 1, $\dptime[\agent][\delta_{\agent}] = t_{j+1} + w_{Loc(\dpid[\agent][\delta_{\agent}]), Loc(\dpid[\agent][\delta_{\agent}-1])} + \rho$ which causes GetPlan($\cdot$) to add the specified wait. If $\dpvar$ reaches the end of the loop past $\numdp - 1$, no new action points will be able to be assigned, so the sequence will be the same.

    Now we consider case 2. In SavePastState($\cdot$), if we instead loop until $\modeloutput(\dptime) \geq t_{j+1}$, $\delta_{\agent} = \dpvar$ for the corresponding action point $\dpvar$. We have already shown an equivalent between durations and action tuple times, and this will now hold until some part of the sequence where $\duration_k(\actionsequence) \geq t_{j+1}$. By construction, this point will exist for an element where $\sequenceelement \in \pickset \cup \dropset$ and there will not be excess waits.
\end{proof}

\subsection{Proof of Lemma \ref{lemma:updated_consistency}: Updated Consistency.} \label{appendix:lemma_4}

Lemma statement: Assume $(result_j = sat) \Rightarrow (\plan_j \models \requirement)$. If $result_{j+1}$ is sat, $\plan_{j+1} \models \allconsistentpredicate$.

\begin{proof}
    If $result_j$ is not sat, the statement is trivially true as above. If $t_{j+1} = 0$, due to SavePastState($t_{j+1}, \encoding, \modeloutput$), $\delta = \mathbf{1}$ (a vector of one's) as when $j=0$, so $\ref{enc:doing_tasks}$ is removed and re-added exactly the same. The rest of the encoding adjustments can be thought of as changing the starting set to be $\taskset_0 = \bigcup \limits_{j'=0}^{j+1} \taskset_{j'}$ and the proof of lemma \ref{lemma:consistent_action_seq} follows.
    
    If $t_{j+1} > 0$, many of the statements of \ref{appendix:lemma_1} still hold. We note that the action tuples saved in SavePastState($\cdot$) correspond to a prefix of each action sequence in the previous plan that is also consistent. An action point may cause $(\wait, t)$ for a $t>0$ value as specified in GetPlan($\cdot$) to be added after the part of the plan that is the same as the previous plan in addition to the pick and drop actions mentioned before. Therefore, $\sequenceelement[\kappa] \in \moveset \Rightarrow \sequenceelement[\kappa + 1] \notin \moveset) \ \forall \kappa = 1, \ldots, k-1$ still holds.

    Let $w_{\kappa} = | \{\sequenceelement[\kappa] \in \waitset | \sequenceelement[\kappa] \in \prefix \} |$. Let $\kappa$ denote the prefix such that $\sequenceelement[\kappa] = \pick_{\task}$ and $\kappa'$ denote the same for the drop. Using the same definitions for $\dpvar'$ and $\dpvar$ as before, $\dpvar' > \dpvar$, so $\kappa = 2\dpvar + w_{\kappa}$ and $\kappa' \geq 2\dpvar + w_{\kappa}$ because the actions for $\dpvar'$ are added after those for $\dpvar$. This is similarly true for showing a task is dropped after it's picked up.

    We have previously shown action point times $\dptime$ are strictly monotonically increasing, even if $\tasksetarrivaltime \neq 0$. The difference here is that $t_{j+1}$ may be different from $t_j$, so we want to remove the previous $\ref{enc:doing_tasks}$ constraints. $\dptime$ are strictly monotonically increasing for $\encoding_j$. Consider saving the state for agent $\agent$ in SavePastState($\cdot$) and that $\delta_{\agent} = \delta$. Given our assumption, the sequence $\dptime[\agent][0], \ldots, \dptime[\agent][\delta - 1]$ will be strictly monotonically increasing. New action point times for $d \geq \delta - 1$ will also be monotonically increasing via $\ref{enc:doing_tasks}$. Therefore, the pick and drop actions will still only occur once.

    Load values are similarly still constrained with the adjustment that load values from some action points will map to three actions instead of two.
\end{proof}

\subsection{Proof of Lemma \ref{lemma:updated_completion}: Updated Completion.} \label{appendix:lemma_5}

Lemma statement: Assume $(result_j = sat) \Rightarrow (\plan_j \models \requirement)$. If $result_{j+1}$ is sat, $\plan_{j+1} \models \allcompletedpredicate[j+1]$ where $\cumtaskset_{j+1} =(\bigcup_{j'=0}^{j+1} \taskset_{j'})$.

\begin{proof}
    If $result_j$ is not sat, the statement is trivially true as above. By lemma $\ref{lemma:updated_consistency}$, each agent's action sequence is consistent. Pretend that $\tasksetarrivaltime = 0$. In lines $\ref{algline:popoff}$ and $\ref{algline:addtasks}$, \ref{enc:doing_tasks} and \ref{enc:at_home_or_doing_task_with_id} are removed and re-added the exact same and \ref{enc:room_UF_task_rooms} - \ref{enc:valid_agent_and_times} are added for $\indexedtaskset$. This encoding is the same as having $\tasksetinit = \cumtaskset_{j+1}$. (We'll disregard that lemmas may have been added during the previous solve as these should only follow from what is in the encoding). We have already shown that such a plan completes the tasks.

    The difference here is from SavePastState($\cdot$) and the introduction of $\tasksetarrivaltime$ in $\ref{enc:doing_tasks}$ and $\ref{enc:valid_agent_and_times}$, which will only affect the duration requirements ($(\duration(\sequenceelement[\kappa_{d}]) \leq \taskdeadline) \wedge (\duration(\sequenceelement[\kappa_{p-2}]) \geq \taskarrivaltime) \ \forall \task \in \cumtaskset_{j+1}$). The requirement on $\tasksetarrivaltime$ in $\ref{enc:valid_agent_and_times}$ is used mainly as a double check. We have previously shown an equivalence between action point times and durations. The previous assignment completed its tasks so the points saved in  SavePastState($\cdot$) will not break $(\duration(\sequenceelement[\kappa_{d}]) \leq \taskdeadline)$ or $(\duration(\sequenceelement[\kappa_{p-2}]) \geq \taskarrivaltime)$. For any new action points, $\ref{enc:doing_tasks}$ will hold, so $(\duration(\sequenceelement[\kappa_{p-2}]) \geq \taskarrivaltime)$ will be true. $\ref{enc:valid_agent_and_times}$ is still required for previous tasks and will be added for new tasks, fulfilling $(\duration(\sequenceelement[\kappa_{p}]) \leq \taskdeadline)$.

\end{proof}

\section{Completeness} \label{appendix:completeness}

\subsection{Proof of Lemma \ref{lemma:dp_max}: Maximum Number of Action Points.} \label{appendix:lemma_6}

As shown in Lemma \ref{lemma:consistent_action_seq}, pick and drop ids will only occur once per task in an agent's assignment in a satisfying $\modeloutput$. By construction, action point ids must be either $\dpid = \agent$ or $\dpid \geq \numagents \wedge \dpid < 2\numtasks + \numagents$. Therefore, for an agent $\agent$, the maximum number of action points that can be assigned to a value other than $\agent$ is $2\numtasks$. Adding in the constrained 0th action point, the total is $2\numtasks + 1 = D_{max}$ where $\dpid = n$ for $d = d' \geq D_{max}$. Therefore, adding an extra action point does not add new free variables, so an unsat result cannot turn sat by adding more action points.

\subsection{Proof of Lemma \ref{lemma:complete_given_max}: Conditional Completeness.} \label{appendix:lemma_7}

Lemma statement: Assume $\Gamma[|\dplist| -1]$. If $D=D_{max}$, Algorithm~\ref{alg:overall} is complete.

\begin{proof}
    We want to show that for each iteration $j$ in Algorithm 1, if given a plan $\plan_{j}$ for which $\plan_{j} \models \Phi^j$, we can find a satisfying assignment $\modeloutput_j$. We will first construct the assignment. Take the action sequence for each agent $\agent$. Set the initial action point $\dptuple[\agent][0] = (\agentid, \initload, \inittime)$. Find the 1-indexed subsequence of indices of pick and drop actions. For each element $k$ in the subsequence, let $\dpid[\agent][k] = 2\taskid + \numagents$ or $\dpid[\agent][k] = 2\taskid + \numagents + 1$ where $\taskid$ is the number of the task that is picked or dropped, respectively. Set $\dptime[\agent][k] = \duration_k(\actionsequence)$ and $\dpload[\agent] = \sequenceload_k(\sequenceelement[k])$. For the task $\taskid$, if $\sequenceelement \in \pickset$, set $\taskstarttime = \duration_k(\actionsequence)$. If $\sequenceelement \in \dropset$, set $\taskendtime = \duration_k(\actionsequence)$. Set $\taskagent = \agentid$. Set all remaining action points to $(\agentid, t_{max}, 0)$.

    We can now go through each constraint in the encoding to check the assignment is indeed satisfying. By constraining the initial point $\ref{enc:init_and_allowed_dp}$ holds (The second part is true by lemma assumption). $\ref{enc:stay_home}$, $\ref{enc:room_UF_agents_and_dist_UF}$, and $\ref{enc:room_UF_task_rooms}$ hold by construction. The left side of $\ref{enc:load_tracking_and_doing_tasks}$ holds by constraining $\dpload[\agent][0]=0$ then each $\dpload$ is equal to the same sum that is used for load in the action sequence. The right side holds because the consistent action sequence definition constrains each pick and drop to only occur once ($(\sequenceelement[\kappa] = \pick_{\task}) \Rightarrow  (\forall \kappa' \neq \kappa \in 1, \ldots  \sequenceindex, \ \sequenceelement[\kappa'] \neq \pick_{\task}$ for pick). We previously showed an equivalence between the action tuple times and durations as an output of GetPlan($\cdot$). The plan we are updating from ($\plan_{j-1}$) is an output of GetPlan($\cdot$). Therefore, when calling SavePastState($\cdot$), we will have two cases as in the proof of lemma $\ref{lemma:updated_plans}$. If the duration of the common prefix $\duration_k(\actionsequence) < \tasksetarrivaltime$ and the new action sequence is exactly the old, all of the unconstrained $\dpid=\agentid$, so \ref{enc:doing_tasks} holds. If the new sequence is not exactly the old, there will be a wait of length $\tasksetarrivaltime - \duration_k(\actionsequence)$, setting the first free time $\dptime[\agent][\delta_{\agent}] = \duration_{k+1}(\actionsequence) = \duration_k(\actionsequence) + (\tasksetarrivaltime - \duration_k(\actionsequence)) + w_{Loc(i^{\agent}_{\delta_{\agent}}),Loc(i^{\agent}_{\delta_{\agent}-1})} + \rho$. The other times similarly follow from the durations. If the common prefix duration $\duration_k(\actionsequence) \geq \tasksetarrivaltime$, $\delta_{\agent}$ will be set such that $t^{\agent}_{\delta_{\agent}-1} \geq \tasksetarrivaltime$. The difference in pairs of durations for the times added for $t^{\agent}_{\delta_{\agent}-1}$ and after will equal the travel time plus $\rho$ as in $\ref{enc:doing_tasks}$, so it will hold for all cases. $\ref{enc:bound_load_bound_id_go_home}$ holds because $\dpload$ is assigned based on loads of the action sequence, which are constrained in the definition of a consistent action sequence. Non-negative ids and $(\dpid = \agentid) \Rightarrow (\dptime = t_{max})$ are true by construction. $\ref{enc:at_home_or_doing_task_with_id}$ follows from the definition of the pick and drop sets and their conversion to the ids of the encoding. $\ref{enc:started_task}$ and $\ref{enc:ended_task}$ follow from construction and the requirements of the consistent action sequence. $\ref{enc:assigned_agent}$ holds because all tasks are constrained to be completed so all $\taskagent$ variables will be set when constructing. $\ref{enc:valid_agent_and_times}$ holds by the definition of a plan being for a set of agents $\agentset$ and the requirements in completed task on when a task is started and when it is ended.

\end{proof}

\subsection{Proof of Lemma \ref{lemma:incr_to_max}: Increment to $D_{max}$.} \label{appendix:lemma_8}

An action point $d$ is free if it is not constrained to have $\dpid = \agentid$.  By inspection we see that the while loop in Lines \ref{line:dp_increase1} and \ref{line:dp_increase2} increases the index into the action point list which changes the assumes until the last one which then places no restriction on the encoding. By construction, the encoding can use all $D_{max}$ action points when no assumptions are present.

\end{document}